\documentclass[11pt]{article}
\usepackage{fullpage}

\usepackage[utf8]{inputenc} 
\usepackage[T1]{fontenc}    
\usepackage{hyperref}       
\usepackage{url}            
\usepackage{booktabs}       
\usepackage{amsfonts}       
\usepackage{nicefrac}       
\usepackage{microtype}      
\usepackage{tcolorbox}
\usepackage{enumerate}
\usepackage{enumitem}
\usepackage[numbers]{natbib}
\usepackage{graphicx} 
\usepackage{caption}
\usepackage{subcaption}
\usepackage{amsmath}
\usepackage{amsthm}
\usepackage{amssymb}
\usepackage{tikz}
\usepackage{tablefootnote}
\usepackage{multirow}
\usepackage{enumerate}
\usepackage{color}
\usepackage{xcolor}
\usetikzlibrary{arrows}

\allowdisplaybreaks[4]

\usepackage{mathrsfs}

\usepackage{algorithm}
\usepackage{algorithmic}
\usepackage{hyperref}
\usepackage{bm,todonotes}

\def\S{\mathcal{S}}
\def\A{\mathcal{A}}

\def\P{\mathbb{P}}

\allowdisplaybreaks

\newtheorem{thm}{Theorem}[section]
\newtheorem{lem}{Lemma}[section]
\newtheorem{cor}{Corollary}[section]
\newtheorem{prop}{Proposition}[section]
\newtheorem{asmp}{Assumption}[section]
\newtheorem{defn}{Definition}[section]

\newtheorem{conj}{Conjecture}[section]
\newtheorem{rem}{Remark}[section]

\usepackage{amsmath,amsfonts,bm}

\def\1{\bm{1}}

\DeclareMathAlphabet{\mathsfit}{\encodingdefault}{\sfdefault}{m}{sl}
\SetMathAlphabet{\mathsfit}{bold}{\encodingdefault}{\sfdefault}{bx}{n}

\def\A{{\bf A}}

\def\E{{\bf E}}

\def\S{{\bf S}}

\def\0{{\bf 0}}
\def\1{{\bf 1}}

\def\AM{{\mathcal A}}

\def\OM{{\mathcal O}}

\def\SM{{\mathcal S}}
\def\TM{{\mathcal T}}

\def\XM{{\mathcal X}}

\def\RB{{\mathbb R}}
\def\EB{{\mathbb E}}

\def\argmax{\mathop{\rm argmax}}
\def\argmin{\mathop{\rm argmin}}

\def\poly{\mathrm{poly}}

\newcommand{\Cphi}{C_{\Phi}}

\newcommand{\propi}{{\textsf{Prop-I}(\hat{\varepsilon})}}
\newcommand{\propii}{{\textsf{Prop-II}(\hat{\varepsilon})}}

\newcommand{\tm}{{\textsf{Time}}}
\newcommand{\alg}{{\textsf{Alg}}}
\newcommand{\hpi}{\hat{\pi}}
\newcommand{\Emu}{ \EB_{s_0 \sim \mu}}
\newcommand{\Etau}{\EB_{\tau \sim \pi} }
\newcommand{\rmdp}{\mathcal{M}(\lambda)}

\title{
Finding the Near Optimal Policy via Adaptive Reduced Regularization in MDPs
}

\author{
Wenhao Yang\thanks{Academy for Advanced Interdisciplinary Studies, Peking University; email: \texttt{yangwenhaosms@pku.edu.cn}. } \\
\and
Xiang Li\thanks{School of Mathematical Sciences, Peking University; email: \texttt{lx10077@pku.edu.cn}. } \\
\and
Guangzeng Xie\thanks{Academy for Advanced Interdisciplinary Studies, Peking University; email: \texttt{smsxgz@pku.edu.cn}. } \\
\and
Zhihua Zhang\thanks{School of Mathematical Sciences, Peking University; email: \texttt{zhzhang@math.pku.edu.cn}. } \\
}

\begin{document}

\maketitle

\begin{abstract}
    Regularized MDPs serve as a smooth version of original MDPs. However, biased optimal policy always exists for regularized MDPs. Instead of making the coefficient $\lambda$ of regularized term sufficiently small, we propose an adaptive reduction scheme for $\lambda$ to approximate optimal policy of the original MDP. It is shown that the iteration complexity for obtaining an $\varepsilon$-optimal policy could be reduced in comparison with setting sufficiently small  $\lambda$. In addition, there exists strong duality connection between the reduction method and solving the original MDP directly, from which we can derive more adaptive reduction method for certain algorithms.
\end{abstract}

\section{Introduction}
Reinforcement learning (RL) has achieved great success empirically, especially when policy  and value function are parameterized by neural networks. Many studies \cite{mnih2015human, schulman2015trust, silver2016mastering, haarnoja2018soft} have shown powerful and striking performance of RL compared to human-level performance. Dynamic Programming \cite{puterman1978modified, scherrer2015approximate,geist2019theory, azar2012dynamic} and Policy Gradient method \cite{williams1992simple, sutton2000policy, kakade2002natural} are the most frequently used optimization tools in these studies. 
However, when policy gradient methods are applied, theoretically understanding the success of RL is still limited in the case that policy is searched either on simplex or parameterized space.
There is a line of recent work \cite{bhandari2020note, agarwal2019optimality, bhandari2019global} on convergence performance of policy gradient methods for MDPs without parameterization, while another line of recent work \cite{mei2020global, cen2020fast, wang2019neural, dai2018sbeed} focus on MDPs with parameterization.

In addition, during the process of learning MDPs, it is often observed that the obtained policy could be quite deterministic while the environment is not fully explored. 
Some prior works \cite{ahmed2019understanding,mnih2016asynchronous, fox2015taming, vamplew2017softmax} propose to impose the Shannon entropy to each reward to make the policy stochastic, so  agent can explore the environment instead of trapping in a local place and achieves success.  Intuitively and empirically speaking, adding entropy regularization helps soften the learning process and encourage agents to explore more, so it might fasten convergence.
However, there are very few works that theoretically analyze the efficiency of regularization. 
\citet{lee2018sparse} proposed the Tsallis entropy \cite{tsallis1988possible} as an alternative choice while \citet{yang2019regularized} presented a more general regularized term and studied the asymptotic property of the optimal regularized policy. 
Except for the choice of regularization term, there are also some studies on the convergence rate of regularized MDPs with certain algorithms. 
For dynamic programming methods, some works \cite{geist2019theory,fox2015taming, schulman2017equivalence} have shown that the convergence rate could be linear. 
As for policy gradient methods, \citet{shani2019adaptive} showed that the convergence rate is sublinear when mirror descent and natural gradient are applied in the Tabular MDP. \citet{mei2020global} showed that the convergence rate is still (near-)linear when softmax parameterization is applied for vanilla gradient descent and \citet{cen2020fast} showed a linear convergence rate in the same setting for natural gradient descent.

However, regularized MDPs always own biased optimal policies due to the  regularization imposed.
The degree of biasedness is often controlled by a regularization coefficient (also called temperature parameter) $\lambda$.
When $\lambda$ approaches zero, the regularized optimal policy converges to the unregularized optimal policy.
Prior works~\cite{cen2020fast,agarwal2019optimality} propose to set the temperature sufficiently small, which we find deteriorates the convergence rate of the algorithms and contradicts with the intention of obtaining faster convergence via regularization. 

Thus, can we design an approach to adaptively  reducing the temperature  such that the output policy is an $\varepsilon$-optimal policy with respect to the original MDP and the convergence speed is maintained or even faster? 
In this paper we give positive answers about the question. In summary, we offer the major contribution as follows.
\begin{itemize}
    \item We propose a reduction method to adaptively tune the temperature parameter $\lambda$.
    Specifically, given the current $\lambda_t$, we resort to a sub-solver algorithm to solve the optimal policy in a $\lambda_t$-regularized MDP within certain accuracy, and then decay $\lambda_t$ to $\lambda_{t+1}$.
    This reduction method allows us to use almost any regularized RL algorithms as long as they satisfy certain policy improvement property (Definition ~\ref{def:prop}).
    We mainly focus on dynamic programming methods and first order gradient methods in our paper.
    \item We show that our algorithm has the same or even lower iteration complexity than simply setting $\lambda$ sufficiently small (it is often the case that $\lambda=O((1-\gamma)\varepsilon/\Cphi)$, see Table~\ref{table:1} for more details).
    In particular, we improve previous analysis on Projected Gradient Ascent without parameterization from $O(\frac{|\SM||\AM|}{\epsilon^2(1-\gamma)^4})$ to $O(\frac{|\SM||\AM|}{\epsilon(1-\gamma)^3})$ when full information of MDP is acquired. Furthermore, we also promote our results to approximate version, where only a $\nu$-restart model can be accessed. With our reduction method, the iteration complexity can be improved from $O(\frac{|\SM||\AM|}{\epsilon^2(1-\gamma)^4})$ to $O(\frac{|\SM||\AM|}{\epsilon(1-\gamma)^3})$, while the sample complexity (number of trajectories) can also be improved from $\widetilde{O}\left(\frac{|\SM|^4|\AM|^3\rho^6\rho_\nu^6}{\varepsilon^6(1-\gamma)^{8}}\right)$ to $\widetilde{O}\left(\frac{|\SM|^4|\AM|^3\rho^5\rho_\nu^6}{\varepsilon^5(1-\gamma)^{7}}\right)$.
    \item 
    We  reveal that our reduction method is a dual approach to solving the unregularized MDP by showing strong duality holds. Thus, we can derive more efficient adaptive reduction methods from dual approach. Though, we conjecture that order of $\frac{1}{\varepsilon}$ is the best we can hope. However, we can still design efficient learning algorithm to reduce dependence on other terms.
\end{itemize}

\begin{table}[htbp!]
  \centering
    \begin{tabular}{|c|c|c|c|}
    \hline
    Method & unregularized MDP & $\lambda=O\left(\frac{(1-\gamma)\varepsilon}{\Cphi}\right)$     & \textsf{AdaptReduce} \\
    \hline
    No parameterization+DP & $O\left(\frac{1}{1-\gamma}\log\frac{C}{\varepsilon(1-\gamma)}\right)$     & $O\left(\frac{1}{1-\gamma}\log\frac{C}{\varepsilon(1-\gamma)}\right)$     & $O\left(\frac{1}{1-\gamma}\log\frac{C}{\varepsilon(1-\gamma)}\right)$ \\[0.2cm]
    \hline
    No parameterization+GD & $O\left(\frac{|\SM||\AM|\rho^2\rho_\nu^2}{\varepsilon^2(1-\gamma)^4}\right)$     & $O\left(\frac{|\SM||\AM|\rho^2\rho_\nu^2 \widetilde{C}_\Phi}{\varepsilon^2(1-\gamma)^4}\right)$     & $O\left(\frac{|\SM||\AM|\rho\rho_\nu^2\widetilde{C}_\Phi}{\varepsilon(1-\gamma)^3}\right)$ \\
    \hline
    \end{tabular}%
    \caption{Iteration complexity for obtaining an $\varepsilon$-optimal policy. $\Cphi$ is upper bound of $\Omega$. $C$ is dependent with $|\SM|, |\AM|,\gamma, \Omega$. $\rho$ and $\rho_\nu$ are upper bound of distribution shift. $\widetilde{C}_\Phi$ is a constant dependent with smoothness of $\Omega$. For Softmax line, $c$ is a constant dependent with MDP, $\mu$ is inital distribution. 
    }
      \label{table:1}
\end{table}%

\subsection{Related Work}

We mainly use dynamic programming  and first order gradient methods as our sub-solvers.
Even considering the existence of regularization, we can still define a (regularized) Bellman operator that still serves as a $\gamma$-contraction~\cite{yang2019regularized,neu2017unified}.
When using the regularized Bellman operator, dynamic programming methods could still achieve a linear convergence rate~\cite{neu2017unified, vieillard2020leverage, smirnova2019convergence}.
\citet{smirnova2019convergence} analyzed the convergence of a general form of regularized policy iteration when $\lambda_t$ decays in a asymptotic sense.
But they did not specify a concrete decay method and only analyzed how regularized policy iteration converges with different asymptotic $\lambda$ decay rate.

In  tabular unregularized MDP scenarios, \citet{agarwal2019optimality} showed that the vanilla policy gradient method achieves a $\widetilde{O}(\frac{1}{\varepsilon^2})$ iteration complexity, while \citet{bhandari2020note} argued that the complexity rate should be  $O(\log\frac{1}{\varepsilon})$ in the same setting. 
However, in order to obtain linear convergence, \citet{bhandari2020note} made use of exact line search.
When the learning rate is appropriately selected (even infinity), the resulting algorithm incorporates policy iteration to a special case, which is shown to converge linearly. 
Since \citet{bhandari2020note} required to select the optimal learning rate, it is not hard to understand their conclusion: a policy output by exact line search performs at least as good as policy iteration.
However, finding a policy good enough by line search needs to setting learning rate optimally or even infinity, which is impractical.


In a regularized MDP, by choosing a sufficiently small and fixed $\lambda$, we show that the iteration complexity of the vanilla policy gradient method~\cite{agarwal2019optimality} with no parameterization is $\widetilde{O}(\frac{1}{\varepsilon^2})$ under the assumption that regularized term is smooth with respect to $\pi$. 
Such an iteration complexity does not have any advantage over its counterpart in unregularized MDPs. 
\citet{shani2019adaptive} analyzed a variant of TRPO and showed that it converges much faster in a regularized MDP.
In particular, it has a $\widetilde{O}(\frac{1}{\varepsilon})$ complexity in a regularized MDP and $\widetilde{O}(\frac{1}{\varepsilon^2})$ in the unregularized case.
We argue that the method can be reduced to regularized policy iteration by choosing learning rate appropriately and thus can achieve a $O(\log\frac{1}{\varepsilon})$ convergence rate.
However, it is quite different from the vanilla policy gradient. In fact, when policy is parameterized by softmax, we show that it is equivalent to natural policy gradient, which can be referred in Section~\ref{sec:dissother}.


For softmax parameterization and vanilla policy gradient method, \citet{mei2020global} showed  for a fixed coefficient $\lambda$ the iteration complexities for unregularized and regularized MDP are $\widetilde{O}(\frac{1}{\varepsilon})$ and $\widetilde{O}(\log\frac{1}{\varepsilon})$, respectively. 
However, we find that the result is somewhat problematic for regularized MDP in Section~\ref{sec:softpg} when $\lambda$ is set sufficiently small. 
Actually, with the result of \cite{mei2020global}, we can not even devise an efficient learning algorithm by adaptively decaying scheme.

\section{Preliminaries and Notation}
\paragraph{Markov Decision Processes.} 

An infinite-horizon MDP is defined by the tuple $(\SM, \AM, P, r, \mu, \gamma)$, where $\SM$ is the state space and $\AM$ the action space, both assumed to be finite with  respective sizes $S$ and $A$. Here
$r: \SM \times \AM \rightarrow [0, R]$ is the bounded reward function. 
Let $\Delta({\XM})$ denote the set of probabilities on $\XM$, that is, $\Delta({\mathcal{X}}) = \{P: \sum_{x \in \mathcal{X}} P(x) =1, P(x) \ge 0\}$. 
Then $P: \SM \times \AM  \rightarrow \Delta(\SM)$ is the unknown transition probability distribution and $\mu \in \Delta(\SM)$ is the initial state distribution.
$\gamma \in [0, 1)$ is the discount factor. 
Let $V^{\pi} \in \RB^{S}$ be the value of a policy $\pi$ with its $s \in \SM$ entry given by $V^{\pi}(s) := \Etau [ \sum_{t=0}^{\infty} \gamma^t r(s_t, a_t)|s_0 = s ]$, where $ \tau \sim \pi$ means the trajectory $\tau = (a_0, s_1, a_1, s_2, a_2, \cdots)$ is generated according to the policy $\pi$.
It is known that $V^{\pi} = (I - \gamma P^{\pi})^{-1}r^{\pi}$ where $[P^\pi]_{s, s'} := \EB_{a \sim \pi(\cdot|s)}P(s'|s, a)$ and $r^{\pi}(s) = \EB_{a \sim \pi(\cdot|s)}r(s, a)$.
For a given initial distribution $\mu$ on $s_0$, we set $V^\pi(\mu) := \Emu V^{\pi}(s)$.

\paragraph{Regularized Markov Decision Processes.} 
Given any convex function $\Omega: \Delta(\AM) \rightarrow \RB$, for any policy $\pi$ and state $s \in \SM$, we  abuse the notation a little bit and denote by $\Omega(\pi, s) := \Omega(\pi(\cdot|s))$ for simplicity.
A regularized MDP can be described by a tuple $(\SM, \AM, P, r, \mu, \gamma, \lambda, \Omega)$, for simplicity, which we denote by $\mathcal{M}(\lambda)$~\cite{geist2019theory}.
We similarly denote the value function of $\pi$ in $\mathcal{M}(\lambda)$ by $V_{\lambda}^{\pi} \in \RB^S$ with its $s \in \SM$ entry given by
\begin{equation}
\label{eq:value}
V_\lambda^\pi(s) =  \EB_{\tau \sim \pi}  \left[  \sum_{t=0}^{\infty}\gamma^t
(r(s_t, a_t) - \lambda \Omega(\pi, s_t)  )\Big|s_0 = s \right].
\end{equation}
The optimal policy is $\pi_{\lambda}^{*} := \argmax_{\pi} J(\pi, \lambda) = \argmax_{\pi} V_{\lambda}^{\pi}(\mu)$ with value function as $V_\lambda^* := V_{\lambda}^{\pi_{\lambda}^*}$.
Sometimes it is convenient to focus on the effect of regularization alone, so we define $\Phi^{\pi} \in \RB^S$ with its $s \in \SM$ entry given by
\begin{equation}
\label{eq:Phi}
\Phi^\pi(s) =  \EB_{\tau \sim \pi} \left[  \sum_{t=0}^{\infty}\gamma^t
\Omega(\pi, s_t)\Big|s_0 = s \right].
\end{equation}
Typically, in $\mathcal{M}(\lambda)$, we consider the sum of discounted reward with regularization as follows:
\begin{equation}
\label{eq:goal}
J(\pi, \lambda) 
= \Emu \EB_{\tau \sim \pi} \sum_{t=0}^{\infty}\gamma^t\left( r(s_t, a_t) - \lambda \Omega(\pi, s_t)\right) 
= V^\pi(\mu) - \lambda \Phi^\pi(\mu),
\end{equation}
where  $\lambda$ is the regularization coefficient (or called temperature) and $ \Phi^\pi(\mu) = \Emu \Phi^\pi(s_0)$.
Previous work shows that it is sometimes more efficient to maximize $J(\pi, \lambda)$ and obtain $\pi_{\lambda}^* =  \argmax_{\pi} J(\pi, \lambda)$~\cite{shani2019adaptive}.
By contrast, unregularized $\mathcal{M}$ focus on $J(\pi) := J(\pi, 0) = V^{\pi}(\mu)$.
We say a policy $\pi$ to be an $\epsilon$-optimal policy  if $V^{*}(\mu) - V^{\pi}(\mu) \le \epsilon$.

\paragraph{Value Function Decomposition.}
The value function of $\pi$ evaluated in $\rmdp$ is $V_{\lambda}^{\pi}$.
When $\lambda = 0$, $V_\lambda^\pi$ will be degenerated to the typical definition of value function $V^{\pi}$. 
We can decompose $V_{\lambda}^\pi$ into two parts, that is,
\begin{equation}
\label{eq:decom}
V_{\lambda}^\pi = V^\pi - \lambda \Phi^\pi.
\end{equation}

\paragraph{Discounted state visitation distribution} of a policy $\pi$ is defined as
\begin{equation}
\label{eq:d}
d_{s_0}^{\pi}(s) : = (1-\gamma) \sum_{t=0}^{\infty} \gamma^t P^{\pi}(s_t = s|s_0),
\end{equation}
where  $P^{\pi}(s_t = s|s_0)$ is the state visitation probability that $s_t = s$, after we execute $\pi$ with initial state $s_0$. 
Again, we overload notation and write: $d_{\mu}^{\pi}(s) = \Emu  d_{s_0}^{\pi}(s)$ and concatenate it as a row vector $d_{\mu}^{\pi} \in \RB^{1 \times S}$.
Interestingly, we have $d_{\mu}^{\pi} = (1-\gamma) \mu(I - \lambda P^{\pi})^{-1}$.
It is often the case that we use another distribution (say, $\nu$) as the initial state distribution used in a RL algorithm, but still use $\mu$ to measure the sub-optimality of our policies.
Given a policy $\pi$, the distribution mismatch coefficient of $\pi$ relative to $\nu$ is defined as $\| \frac{d_{\mu}^{\pi}}{\nu}\|_{\infty}$.

\paragraph{Bregman Divergence.} Given any strictly convex and continuously differentiable function $\Omega: \Delta(\AM) \rightarrow \RB$, for any two policies $\pi, \pi' \in \Delta(\AM)^S$ and $s \in \SM$, the Bregman divergence between $\pi, \pi'$ at $s$ is defined as
$  D_{\Omega}(\pi'||\pi)(s) := D_{\Omega}(\pi'(\cdot|s)\| \pi(\cdot|s)) = \Omega(\pi'(\cdot|s)) - \Omega(\pi(\cdot|s)) - \langle \nabla_{\pi(\cdot|s)}\Omega(\pi(\cdot|s)), \pi' - \pi \rangle $.
For simplicity, we let $D_{\Omega}(\pi'||\pi) \in \RB^S$ with its $s \in \SM$ entry given by $ D_{\Omega}(\pi'||\pi)(s)$.

\paragraph{Performance of optimal regularized policy.} Given a fixed $\lambda$, the performance of optimal regularized policy $\pi_\lambda^*$ is guaranteed by the following proposition\cite{cen2020fast, yang2019regularized}.
\begin{prop}
    \label{prop:bias}
	Denote $\pi_\lambda^*\in\argmax_\pi J(\pi,\lambda)$ and $V^*(s)=\max_\pi V^{\pi}$. Then
	\begin{align}
		\|V^*-V^{\pi_\lambda^*}\|_{\infty}\le\frac{\lambda}{1-\gamma}C_\Phi,
	\end{align}
	where $C_\Phi = \max_{\pi\in\Delta(\AM)}\Omega(\pi)$.
\end{prop}

\paragraph{Controlling the bias.} Regularized MDP often has a biased optimal policy as shown in Proposition~\ref{prop:bias}. In order to find an $\varepsilon$-optimal policy with respect to original unregularized MDP, many prior works propose to fix a sufficiently small $\lambda=O((1-\gamma)\varepsilon/\Cphi)$. However, in some cases, the convergence rate could be much slower than it when $\lambda=\Theta(1)$. Thus we propose an adaptive reduction method to help control the bias in the next section.

\section{Methodology}
\label{sec:reduction}

\begin{algorithm}[htb]
\caption{The \textsf{AdaptReduce}}
\label{alg:ar}
\begin{algorithmic}
	\STATE {\bfseries Input:}  $T$--- the number of epochs, $\lambda_0$--- an initial regularization parameter, an algorithm $\alg$ that tries to produce the optimal policy of $J(\pi, \lambda)$ for a given $\lambda$.
	\FOR{iteration $t=0$ {\bfseries to} $T-1$}
	\STATE $ \hpi_{t+1} \leftarrow \alg(\hpi_{t}, \lambda_t)$;
	\STATE $\lambda_{t+1} \leftarrow \frac{\lambda_{t}}{2}$;
	\ENDFOR
	\STATE {\bfseries Return: $\hpi_T$} 
\end{algorithmic}
\end{algorithm}

In this section we propose \textsf{AdaptReduce} (Algorithm~\ref{alg:ar}) to control the bias from $\lambda$. \textsf{AdaptReduce} works in the following way. At the beginning of \textsf{AdaptReduce}, we set $\hat{\pi}_0$ as any given initial policy.
At each iteration $t= 0 ,1, \cdots, T-1$, we first focus on finding the optimal policy in the regularized MDP $\mathcal{M}(\lambda_{t})$ and then update the value of regularization coefficient $\lambda_t$.
Specifically, we run $\alg$ with starting policy $\hat{\pi}_t$ in each iteration, and let the output be $\hat{\pi}_{t+1}$. 
After all $T$ iterations are finished, \textsf{AdaptReduce} simply outputs $\hat{\pi}_T$.
Here we do not aim to solve out the optimal policy in $\mathcal{M}(\lambda_{t})$ exactly, because our target is to find the original optimal policy $\pi^*$ and the reason we prefer to optimize in a regularized MDP is the benefit (like faster convergence~\cite{shani2019adaptive} and better exploration) it would offer.

\subsection{Convergence Analysis}

\begin{defn}
    \label{def:prop}
    We say an algorithm $\alg(\pi_0,\lambda, \varepsilon,\hat{\varepsilon})$ maximizing $J(\pi,\lambda)$ satisfies approximate convergence with given accurace $\varepsilon$ property  in time $\tm(\lambda)$ if, for every starting policy $\pi_0$, it produces $\pi_1\leftarrow\alg(\pi_0,\lambda,\varepsilon,
    \hat{\varepsilon})$ such that $V_{\lambda}^*-V_{\lambda}^{\pi_1}\le\varepsilon+\hat{\varepsilon}$.
\end{defn}

In this paper, we mainly consider two types of $\alg(\pi_0,\lambda,\varepsilon,\hat{\varepsilon})$: (i) $\varepsilon=\frac{1}{4}(V_\lambda^{*}(\mu)-V_\lambda^{\pi_0}(\mu))$; (ii) At timestep $t$, $\alg(\pi_t,\lambda_t,\varepsilon_t,\hat{\varepsilon})$ outputs $\pi_{t+1}$ with $\varepsilon_t=\frac{\lambda_0}{2^t}\frac{\Cphi}{1-\gamma}$. We denote each property as $\propi$ and $\propii$ respectively.

\begin{rem}
The requirement of the $\propi$ is mainly twofold: (a) the first part is homogenous contraction convergence, i.e., the value function of the output policy $\pi_i$ will be closer to that of the optimal policy $\pi_{\lambda}^*$ than the starting policy $\pi_0$ by at least a factor of $1/4$;\footnote{The factor $1/4$ has nothing special and can be replaced by any number in $(0, 1)$.} (b) the second is the relaxation for the closeness, an approximation error no larger than $\hat{\varepsilon}$ allowed. However, $\propi$ means $\alg$ is able to contract the initial error by a fixed factor in a given time, which, however, is barely met for accurate controlling. 
In the case, it is better to require $\alg$ solving $\pi_{\lambda}^*$ within an absolute error rather than contracting initial errors.
Thus, we are motivated to propose another type of stopping criteria ($\propii$) to deal with some sub-solvers that can't satisfy $\propi$ property.
\end{rem}

\begin{rem}
    In Definition~\ref{def:prop}, we give two types of error $\varepsilon$ and $\hat{\varepsilon}$. In this paper, $\varepsilon$ denotes the desired accuracy we would like to achieve. While $\hat{\varepsilon}$ denotes the computation and statistical error evolving from sub-solver.
\end{rem}

\begin{rem}
	$\tm(\lambda)$ denotes the time $\alg$ needs to produce an output policy satisfying the $\propi$ or $\propii$ property.
	In different contents, each unit of time may have different meanings. For example, a unit of time may mean a generated sample or an evaluation of policy gradient.
	In the following, we will specify the meaning of a time unit.
\end{rem}

\begin{asmp}[Bounded strongly-convex regularization]
	\label{ass:reg}
	Assume regularization function $\Omega: \Delta(\AM) \rightarrow \RB^+ \cup \{0\}$ is (i) $1$-strongly convex with respect to $\|\cdot\|$, that is, for any two policies $\pi', \pi$, 
		\[  \Omega(\pi') \ge  \Omega(\pi) + \langle \nabla  \Omega(\pi), \pi'-\pi \rangle + \frac{1}{2} \|\pi'-\pi\|^2.   \]
	and (ii) uniformly bounded, that is, $\max_{\pi \in \Delta(\AM)} \Omega(\pi) \le \Cphi$.
\end{asmp}

\begin{thm}
	\label{thm:reduction}
	Let Assumption~\ref{ass:reg} hold, $\lambda_0$ the initial regularization parameter and $D_0 = V_{\lambda_0}^*(\mu) - V_{\lambda_0}^{\hat{\pi}_0}(\mu).$ 
	For any $T$, \textsf{AdaptReduce} with any solver algorithm satisfying the $\propi$  property will produce a policy $\hat{\pi}_T$ satisfying 
	\[   V^*(\mu)  -V^{\hat{\pi}_T}(\mu) \le \frac{D_0}{4^T} + \frac{4}{3} \hat{\varepsilon} +  \frac{6\lambda_0\Cphi}{1-\gamma}  \frac{1}{2^T}.\]
\end{thm}

\begin{cor}
	\label{cor:reduction}
	In the same setting of Theorem~\ref{thm:reduction}, \textsf{AdaptReduce} with any solver algorithm satisfying the $\propi$ property where $\hat{\varepsilon} = 
	O( 	 \epsilon )$ will produce an $\epsilon$-optimal policy in iteration $T = O(\log \frac{C_0}{\epsilon(1-\gamma)})$ and time $\sum_{t=0}^{T-1} \tm(\frac{\lambda_0}{2^t})$ where $C_0 = D_0 + \frac{\lambda_0 \Cphi}{(1-\gamma)}$.
\end{cor}


\begin{thm}
	\label{thm:reduction2}
	Let Assumption~\ref{ass:reg} hold, $\lambda_0$ the initial regularization parameter and $D_0 = V_{\lambda_0}^*(\mu) - V_{\lambda_0}^{\hat{\pi}_0}(\mu).$ 
	For any $T$, \textsf{AdaptReduce} with any solver algorithm satisfying $\propii$ will produce a policy $\hat{\pi}_T$ satisfying 
	\[   V^*(\mu)  -V^{\hat{\pi}_T}(\mu) \le  \frac{6\lambda_0}{2^T}\frac{C_\phi}{1-\gamma}+\hat{\varepsilon}.\]
\end{thm}



\subsection{Examples of $\alg$}
\label{sec:example}
There are many choices of the sub-solver $\alg$.
In the section, we specify two important classes of methods, one based on dynamic planning and the other on policy gradient.
There of course exist other powerful optimization methods that suit to do the task (for example, see~\cite{agarwal2019optimality} and reference therein). 
Here we just take those of most interest for example.

\subsubsection{Dynamic Planning Based $\alg$: Regularized Modified Policy Iteration}
\label{sec:egrmpi}
We define the regularized Bellman operator $\TM_{\lambda}^{\pi}: \RB^S \to \RB^S$ with $s \in \SM$ entry given by
\begin{equation}
\label{eq:bellman}
[\TM_{\lambda}^{\pi} V](s) =  \EB_{a \sim \pi(\cdot|s)}(r(s, a) + \gamma \EB_{s' \sim P(\cdot|s, a)} V(s')) - \lambda \Omega(\pi(\cdot|s))  
\end{equation}
One can show that the regularized Bellman operators have the same properties as the classical ones, so \cite{geist2019theory} apply classical dynamic programming to solve the regularized MDP problem.
They proposed regularized regularized MPI (RMPI) by modifying a classic dynamic programming method, Modified policy iteration (MPI)~\cite{puterman1978modified}.

\begin{algorithm}[htb]
	\caption{Regularized Modified Policy Iteration (RMPI)}
	\label{alg:RMPI}
	\begin{algorithmic}
		\STATE {\bfseries Input:} an initial policy $\alpha_0$, a regularization parameter $\lambda$ and $K$ the number of iteration.
		\FOR{iteration $k=0$ {\bfseries to} $K-1$}
		\STATE  find $\alpha_{k+1} $ that satisfies $ \max_{\pi \in \Delta(\AM)}\TM_{\lambda}^{\pi} V_k \le \TM_{\lambda}^{\alpha_{k+1}} V_k + \epsilon_01_S$ point-wisely;  
		\STATE  $V_{k+1} = \left( \TM_{\lambda}^{\alpha_{k+1}} \right)^m V_k + \epsilon_01_{S} $
		\ENDFOR
		\STATE {\bfseries Return: $\alpha_K$} 
	\end{algorithmic}
\end{algorithm}

\begin{thm}
	\label{thm:RMPI}
	Define the concentrability coefficients as
	\[
	C_{\infty}^{i}=\frac{1-\gamma}{\gamma^{i}} \sum_{j=i}^{\infty} \gamma^{j} \max _{\pi_{1}, \ldots, \pi_{j}}\left\|\frac{\nu P_{\pi_{1}} P_{\pi_{2}} \ldots P_{\pi_{j}}}{\mu}\right\|_{\infty}
	\ \text{and} \
	C = \max_{i \ge 0} C_{\infty}^{i},
	\]
	then RMPI (Algorithm~\ref{alg:RMPI}) satisfies the $\propi$ property with $\hat{\varepsilon}=\frac{(1+2\gamma)C}{(1-\gamma)^2}\epsilon_0$ and $  \tm(\lambda) = \ln\frac{8C}{1-\gamma}/ \ln\frac{1}{\gamma}$,
	where each unit time means one iteration processed to update the policy.
\end{thm}

Theorem~\ref{thm:RMPI} shows that RMPI have $\propi$ property; here $\hat{\epsilon}$ results from the numerical error $\epsilon_0$.
Besides, the time $ \tm(\lambda) = \ln\frac{8C}{1-\gamma}/ \ln\frac{1}{\gamma}$ has nothing to do with $\lambda$; indeed, for a sequence that is generated by a compressed mapping, it needs only a constant number of iterations to get a constant factor close to the fix point\footnote{Formally speaking, let $\mathcal{T}$ be a $\gamma$-compressed mapping on $\RB^S$, then it has a fix point $X^* = \TM X^* \in \RB^S$. Consider the sequence $X_k = \TM X_{k-1}$, then we have $\|X_k - X^*\| =\| \TM ( X_{k-1} - X^* ) \| \le \gamma \| X_{k-1} - X^* \|$. Hence as long as $k \ge \log 4/\log \gamma$, we have $\|X_k - X^*\| \le \frac{1}{4} \|X_0 - X^*\|$.}.
The fact that the regularized Bellman operator is a $\gamma$-contraction (see Proposition 2 in~\cite{geist2019theory}) makes $ \tm(\lambda)$ unrelated with $\lambda$.

\begin{cor}
	\label{cor:RMPI}
		Given an accuracy $\epsilon$, if $\epsilon_0$ is sufficiently small such that $\epsilon_0  \le  \frac{(1-\gamma)^2}{(1+2\gamma)C} \cdot \epsilon $, then $\alg$ with RMPI serving as the sub-solver will produce an $\epsilon$-optimal policy in outer iteration $T = O(\log\frac{C_0}{\epsilon})$ and in time (total iteration)
	\[  O\left(   \log \frac{8C}{1-\gamma} \log\frac{C_0}{\epsilon}/\log \frac{1}{\gamma} \right) = O\left( \frac{1}{1-\gamma} \log \frac{8C}{1-\gamma}  \log\frac{C_0}{\epsilon} \right)  \]
	with $ C_0 = V_{\lambda_0}^*(\mu) - V_{\lambda_0}^{\hat{\pi}_0}(\mu) +
	\frac{\lambda_0 \Cphi}{(1-\gamma)}$.
\end{cor}

The combination of Corollary~\ref{cor:reduction} and Theorem~\ref{thm:RMPI} leads to Corollary~\ref{cor:RMPI} directly.
When $\lambda = 0$, RMPI is reduced to MPI, the latter still satisfying the $\propi$ property.
For sake of simplicity, let assume $\epsilon_0=0$ without loss of generality\footnote{Otherwise, we can set $\epsilon_0$ sufficiently small such that $\hat{\epsilon} = O(\epsilon)$, then the effect of $\hat{\epsilon}$ is ignorable.}.
In order to reach a $\epsilon$-optimal policy, MPI needs about $O(\frac{1}{1-\gamma}\log\frac{1}{\epsilon})$ iterations.
Interestingly, RMPI needs the same number of iterations up to only a constant factor, which implies there is no harm to use RMPI instead of MPI.

\begin{rem}
	Note that RMPI alone converges linearly, which means one step of RMPI is already able to improve the current policy considerably. 
	When applying RMPI as a sub-solver for \textsf{AdaptReduce}, we run $\widetilde{\OM}(1)$ steps of RMPI between consecutive $\lambda$ decays.
	We find that this can be relaxed into a more simplified form, where the policy $\pi$ and coefficient $\lambda$ can be updated alternatively as Algorithm~\ref{alg:RMPI-r} describes. 
	Under this scheme, we show in Theorem~\ref{thm:rmpi-r} that the convergence rate can be imporved to $O\left(\frac{1}{1-\gamma}\log\frac{C}{\varepsilon(1-\gamma)}\right)$, a logarithmic term removed.
\end{rem}
\begin{thm}
	\label{thm:rmpi-r}
	Given an accuracy $\varepsilon$ and assuming $\frac{1}{2}\le\gamma^2$, if $\varepsilon_0$ satisfies $\varepsilon_0\le\frac{(1-\gamma)^2}{6}$, the output $\pi_T$ of Algorithm~\ref{alg:RMPI-r} is an $\epsilon$-optimal policy after $O\left(\frac{1}{1-\gamma}\log\frac{C}{\varepsilon(1-\gamma)}\right)$ iterations.
\end{thm}
\begin{algorithm}[htb!]
	\caption{Regularized Modified Policy Iteration (RMPI) with Alternative Reduction}
	\label{alg:RMPI-r}
	\begin{algorithmic}
		\STATE {\bfseries Input:} an initial value function $V_0$ and policy $\pi_0$, a regularization parameter $\lambda_0$ and $T$ the number of iteration.
		\FOR{iteration $t=0$ {\bfseries to} $T-1$}
		\STATE $V_{t+1}=\TM_{\lambda_t}^{\pi_t}V_t+\varepsilon_0 1_{\S}$
		\STATE $\lambda_{t+1}=\frac{1}{2}\lambda_t$
		\STATE find $\pi_{t+1}$ that satisfies $\max_{\pi\in\Delta(\AM)}\TM_{\lambda_{t+1}}^{\pi}V_{t+1}\le\TM_{\lambda_{t+1}}^{\pi_{t+1}}V_{t+1}+\varepsilon_0 1_{S}$
		\ENDFOR
		\STATE {\bfseries Return: $\pi_T$} 
	\end{algorithmic}
\end{algorithm}

\subsubsection{Policy Gradient Based $\alg$: Projected Gradient Ascent}
Projected Gradient ascent method is a practical and popular RL method. 
\citet{agarwal2019optimality} showed that projected gradient ascent can converge to $\varepsilon$-optimal policy with $\widetilde{O}(\frac{1}{\varepsilon^2})$ iterations in tabular un-regularized MDP. 
By similar proof techniques, we argue that regularized MDP with projected gradient ascent (Algorithm~\ref{alg:pga}) can also converge with rate $\widetilde{O}(\frac{1}{\varepsilon^2})$ as Theorem~\ref{thm:pga} shows. Before we provide detailed convergence results, we have to clarify the policy optimization oracle. Even though we would like to obtain a near optimal policy w.r.t. initial distribution $\mu$, but we always only own access with $\nu$-restart model \cite{shani2019adaptive, kakade2002approximately}. Fortunately, for Tabular MDP without parameterization, the optimal policy $\pi_\lambda^*$ is simultaneously optimal for all starting $s$ \cite{bellman1959functional, yang2019regularized}. Thus, we have Theorem~\ref{thm:diffval} to control value difference under different initial distributions. To simplify, we always assume $\min_s\nu(s)>0$ while $\mu\ll\nu$ is already enough.

\begin{thm}
\label{thm:diffval}
    For any two initial distributions $\mu,\nu$ satisfying $\min_s \nu(s)>0$, the following inequality holds for any $\pi$:
    \begin{align}
        J_\mu(\pi_\lambda^*,\lambda)-J_\mu(\pi,\lambda)\le\left\|\frac{\mu}{\nu}\right\|_\infty\left(J_\nu(\pi_\lambda^*,\lambda)-J_\nu(\pi,\lambda)\right)
    \end{align}
    where $\left\|\frac{\mu}{\nu}\right\|_\infty=\max_s\frac{\mu(s)}{\nu(s)}$.
\end{thm}


\begin{asmp}[Bounded distribution mismatch]
	\label{ass:ubcc}
	Assume $0<\left\|\frac{\mu}{\nu}\right\|\le\rho$.
\end{asmp}

\begin{asmp}[Uniform bounded $\nu$-restart model]
    \label{ass:ubnu}
    Assume $\left\|\frac{d_{\pi^*_\lambda,\nu}}{d_{\pi,\nu}}\right\|_{\infty}\le\rho_\nu$ for all $0\le\lambda\le\lambda_0$ and $\pi\in\Delta(\AM)^\SM$.
\end{asmp}

In Assumption~\ref{ass:ubcc}, $\rho$ measures distribution mismatch between target initial distribution $\mu$ and behavior initial distribution $\nu$. In Assumption~\ref{ass:ubnu}, $\rho_\nu$ measures how uniform of $\nu$ is. If $\nu$ is a uniform distribution on $\SM$, $\rho_\nu=1$ is obtained. Also, $\rho_\nu$ can be upper bounded by $\frac{1}{(1-\gamma)\min_s \nu(s)}$ trivially. These two assumptions are slightly different than the usual distribution mismatch $\|d_{\pi^*_\lambda,\mu}/\nu\|_{\infty}$ mentioned in \cite{kakade2002approximately, bhandari2019global, shani2019adaptive, agarwal2019optimality}. In fact, the true distribution shift to be bounded is $\|d_{\pi_\lambda^*,\mu}/d_{\pi,\nu}\|_\infty$ for any policy $\pi$. \citet{kakade2002approximately} bound it with $\left\|\frac{d_{\pi^*_\lambda,\mu}}{(1-\gamma)\nu}\right\|_{\infty}$. We argue we can decompose this term into two parts:
\begin{align*}
    \left\|\frac{d_{\pi_\lambda^*,\mu}}{d_{\pi,\nu}}\right\|_\infty=\left\|\frac{d_{\pi_\lambda^*,\mu}}{d_{\pi_\lambda^*,\nu}}\cdot\frac{d_{\pi_\lambda^*, \nu}}{d_{\pi,\nu}}\right\|_\infty\le\left\|\frac{d_{\pi_\lambda^*,\mu}}{d_{\pi_\lambda^*,\nu}}\right\|_\infty\left\|\frac{d_{\pi_\lambda^*, \nu}}{d_{\pi,\nu}}\right\|_\infty\le\rho\rho_\nu
\end{align*}
where the last inequality holds by definition of $d_{\pi,\mu}$.

\begin{rem}
In \cite{agarwal2019optimality}, the convergence result of Projected Gradient Ascent is given for the minimizer of previous $T$ steps, i.e., $\min_{t=0,...,T-1}J_\mu(\pi^*,0)-J_\mu(\pi_{t+1},0)$.
However, we can't identify where the minimizer indexes, since when we have no knowledge about the target initial distribution $\mu$, the inconsistency between $\mu$ and actual initial distribution $\nu$ disables us to evaluate the value of $J_\mu(\pi_{t},0)$ for a given $t$.
With alternative assumptions~\ref{ass:ubcc} and~\ref{ass:ubnu}, Theorem~\ref{thm:diffval} guarantees that the optimality measured by initial distribution $\nu$ implies that measured by $\mu$.
Hence, an $\varepsilon/\rho$-optimal policy w.r.t. $\nu$ is also an $\varepsilon$-optimal policy w.r.t. $\mu$.
\end{rem}

\begin{thm}[Convergence of Projected Gradient Ascent]
	\label{thm:pga}
	Assume $\Omega(\pi)$ is smooth, $\lambda$ is fixed and Assumption~\ref{ass:ubnu} holds, we can obtain an policy $\pi_T$ satisfying:
	\begin{align}
		J_\nu(\pi_\lambda^*,\lambda)-J_\nu(\pi_T, \lambda)\le4\rho_\nu\sqrt{|\SM|}\left(\sqrt{\frac{2L_\lambda(J_\nu(\pi^*_\lambda,\lambda)-J_\nu(\pi_0,\lambda))}{T}}\right)
	\end{align}
	where $L_\lambda$ is the smoothness of $J_\nu(\pi,\lambda)$.
\end{thm}

\begin{cor}
    \label{cor:pga}
    Under the same setting with Theorem~\ref{thm:pga} and Assumption~\ref{ass:ubcc} holds, we obtain an policy $\pi_T$ satisfying:
    \begin{align}
        J_\mu(\pi_\lambda^*,\lambda)-J_\mu(\pi_T, \lambda)\le4\rho\rho_\nu\sqrt{|\SM|}\left(\sqrt{\frac{2L_\lambda(J_\nu(\pi^*_\lambda,\lambda)-J_\nu(\pi_0,\lambda))}{T}}\right)
    \end{align}
\end{cor}

\begin{algorithm}[htb]
	\caption{Projected Gradient Ascent}
	\label{alg:pga}
	\begin{algorithmic}
		\STATE {\bfseries Input:} an initial policy $\pi_0$, a regularization parameter $\lambda$ and $T$ the number of iteration.
		\FOR{iteration $t=0$ {\bfseries to} $T-1$}
		\STATE $\pi_{t+1} = \arg\min_{\pi \in {\Delta(\mathcal{\AM})}^{S}} 
		\left[ \langle - \nabla_{\pi} J_\nu(\pi_t,\lambda),\pi - \pi_{t}\rangle  
		+ \frac{1}{2\eta}\|\pi-\pi_t\|_{2}^{2}
		\right]$
		\ENDFOR
		\STATE {\bfseries Return: $\pi_T$} 
	\end{algorithmic}
\end{algorithm}

\begin{rem}
	\label{rem:lsmooth}
	We give an intuition for Theorem~\ref{thm:pga}. By \citet{agarwal2019optimality} and assumption that $\Omega(\pi)$ is $L$-smooth, it's guaranteed that $V_{\lambda}^{\pi}$ is also $L$-smooth. In this case, we can always find a first-order stationary point with rate $O(\frac{1}{\varepsilon^2})$. \citet{agarwal2019optimality} shows that stationary point is also a global optimal point, which concludes Theorem~\ref{thm:pga}. But note that the assumption $\Omega(\pi)$ is $L$-smooth is quite strong. For example, $\Omega(\pi)=\sum_{a}\pi(a|s)\log\pi(a|s)$ breaks the assumption while other `sparse' regularization satisfies this assumption. Can we remove the smoothness assumption and still obtain a convergence result? We conjecture it's possible and leave it as future work.
\end{rem}

Now Theorem~\ref{thm:pga} says the value function can be improved with rate $\widetilde{O}\left(\sqrt{\Delta/T}\right)$, where $\Delta=J_\nu(\pi_{\lambda}^*,\lambda)-J_\nu(\pi_0,\lambda)$. If we let the sub-solver satisfies $\propi$ property, then $\tm(\lambda)=\widetilde{O}\left(\frac{1}{\Delta}\right)$. But note that $\Delta$ could be exponentially small, thus we can't lower bound $\Delta$, which means the sub-solver can't solve it in time independent with $\Delta$. Thus we relax the $\propi$ property to argue the sub-solver satisfies $\propii$ property. By Corollary~\ref{cor:pga}, to obatin an $\varepsilon$-optimal policy, the iteration complexity could be $\widetilde{O}\left(\frac{|\SM||\AM|\rho^2\rho_\nu^2\widetilde{C}_\Phi}{\varepsilon^2(1-\gamma)^4}\right)$. But, with Algorithm~\ref{alg:pga} as a sub-solver in Algorithm~\ref{alg:ar}, the iteration complexity could reduce to $O\left(\frac{|\SM||\AM|\rho\rho_\nu^2\widetilde{C}_\Phi}{\varepsilon(1-\gamma)^3}\right)$ as Theorem~\ref{thm:pgared} implies.

\begin{cor}[Time of Projected Gradient Ascent]
    \label{cor:pgaahcc}
	Under the same setting in Theorem~\ref{thm:pga}, at timestep $t$ of Algorithm~\ref{alg:ar} with $\propii$ holding, the iteration number of Projected Gradient Ascent is at most:
	\begin{align}
		\tm(\lambda_t)=\frac{128|\SM|\rho_\nu^2 L_{\lambda_t}(1-\gamma)}{\lambda_0 C_\phi}2^t
	\end{align}
\end{cor}

\begin{thm}
	\label{thm:pgared}
	Under the same assumptions in Theorem~\ref{thm:pga}, $\alg$ with Projected Gradient Ascent serving as the sub-solver will produce an $\varepsilon$-optimal policy w.r.t. initial distribution $\mu$ in iteration $T=O\left(\log\frac{6\rho\lambda_0 \Cphi}{\varepsilon(1-\gamma)}\right)$ in total time $O\left(\frac{|\SM||\AM|\rho\rho_\nu^2\widetilde{C}_\Phi}{\varepsilon(1-\gamma)^3}\right)$, where $\widetilde{C}_\Phi$ is dependent with upper bounds of $|\Omega|, \|\nabla\Omega\|_\infty$ and $\|\nabla^2\Omega\|_\infty$.
\end{thm}

In reality, it's unrealistic to obtain the exact value of gradient at each time-step. 
A more practical way is to estimate gradients with sufficient samples.
In this scenario, convergence is guaranteed with high probability when we have enough samples. 
With given $\nu$-restart model, we always start from an initial state $s_0\sim\nu$. 
Thus, by interacting with MDP under a given policy $\pi$, we can obtain independent trajectories and approximate the value of $J_\nu(\pi,\lambda)$. 
However, to obtain an estimation of gradient $\nabla J_\nu(\pi,\lambda)$, we need to obtain an initial state $s_0\sim d_{\pi,\nu}$ and start interacting with MDP, which is unattainable when transition probability is unknown. 
In \citet{kakade2002approximately, shani2019adaptive}, an approximate method is proposed.
In particular, we draw a start state $s$ from $\nu(s)$, and accept it as initial state with probabilty $1-\gamma$. 
Otherwise, we transits it to next state with probability $\gamma$. 
We repeat the process until an acceptance is made. 
In process of one trajectory sampling, we also stop interacting when the length of this trajectory achieves $K=\widetilde{O}\left(\frac{1}{1-\gamma}\right)$, which is known as effective horizon. The details of sampling scheme can be referred in Appendix~\ref{apx:sample}.

\begin{thm}[Convergence of Projected Gradient Ascent with Sampling]
\label{thm:pga_sampling}
Under the same setting in Theorem~\ref{thm:pga} and truncated length of trajectories being $K=\log\frac{12(1+\lambda\Cphi)}{\varepsilon(1-\gamma)^2}/\log\frac{1}{\gamma}$, with probability $1-\delta$ and number of trajectories $\widetilde{O}\left(\frac{|\SM||\AM|^2(1+\lambda\widetilde{C}_\Phi)^2}{\varepsilon^2(1-\gamma)^4}\right)$ at each time-step $t$, we have:
\begin{align}
\min_{t=0,...,T-1}J_\nu(\pi_{\lambda}^*,\lambda)-J_\nu(\pi_{t+1},\lambda)\le4\rho_\nu\sqrt{|\SM|}\sqrt{\frac{L_{\lambda}(J(\pi_\lambda^*,\lambda)-J(\pi_0, \lambda))}{T}+\varepsilon}
\end{align}
In order to make the LHS equaling with $\varepsilon'/\rho$, which guarantees $\varepsilon'$-optimal policy w.r.t. initial distribution $\mu$, the total iteration complexity could be $T=O\left(\frac{|\SM||\AM|\rho\rho_\nu^2(1+\lambda\widetilde{C}_\Phi)^2}{\varepsilon'^2(1-\gamma)^4}\right)$ and the total number of sampled trajectories is $\widetilde{O}\left(\frac{|\SM|^4|\AM|^3\rho^6\rho_\nu^6(1+\lambda\widetilde{C}_\Phi)^4}{\varepsilon'^6(1-\gamma)^{8}}\right)$.
\end{thm}

In Theorem~\ref{thm:pga_sampling}, the convergence is guaranteed with high probability. Suppose we could obtain $\pi_{\widehat{t}}$ attaining minimum of RHS of Theorem~\ref{thm:pga_sampling}, we can take Projected Gradient Ascent with Sampling as a sub-solver in Algorithm~\ref{alg:ar}, which also accelerates iteration and sample complexity by a factor of $\frac{1}{\varepsilon(1-\gamma)}$ as following Theorem~\ref{thm:pgared_sampling} shows.

\begin{thm}
    \label{thm:pgared_sampling}
    Suppose we can obtain $\widehat{\pi}_{t+1}$ attaining minimum of $J_\nu(\pi_{\lambda_t}^*,\lambda_t)-J_\nu(\pi_{k+1}, \lambda_t)$ over $k=0,...,T_t-1$ at each time-step $t$ in Algorithm~\ref{alg:ar}, where $T_t=\tm(\lambda_t)$. Under the same setting in Theorem~\ref{thm:pga_sampling}, $\alg$ taking Projected Gradient Ascent with sampling as sub-solver will produce an $\varepsilon$-optimal policy w.r.t initial distribution $\mu$, with probability $1-\delta$, the total iteration complexity is $O\left(\frac{|\SM||\AM|\rho\rho_\nu^2\widetilde{C}_\Phi}{\varepsilon(1-\gamma)^3}\right)$ and the number of trajectories is $\widetilde{O}\left(\frac{|\SM|^4|\AM|^3\rho^5\rho_\nu^6(1+\lambda_0\widetilde{C}_\Phi)^3}{\varepsilon^5(1-\gamma)^{7}}\right)$.
\end{thm}

Theorem~\ref{thm:pgared_sampling} argues we can find a $\widehat{\pi}_{t+1}$ attaining minimum difference value in sub-solver. In practice, we can use $\{\pi_{i+1}\}_{i=0}^{T_t-1}$ to simulate another trajectories to evaluate $J(\pi_{i+1},\lambda_t)$, which can be referred in Appendix~\ref{apx:sample}. As long as we have enough samples, it's guaranteed that $\widehat{J}(\pi_{i+1},\lambda_t)\approx J(\pi_{t+1},\lambda_t)$ as Theorem~\ref{thm:minimum} shows. Compared with gradient estimation, which requires $\widetilde{O}\left(\frac{1}{\varepsilon^4(1-\gamma)^4}\right)$ trajectories for each policy, value estimation requires less samples $\widetilde{O}\left(\frac{1}{\varepsilon^2(1-\gamma)^2}\right)$ than it to achieve same statistical error level. Thus, we can select an policy $\widehat{\pi}_{t+1}$ attaining maximum estimated value $\widehat{J}(\pi_{k+1},\lambda_t)$ over $k=0,...,T_t-1$ and take it as initial policy in next time-step $t+1$.

\begin{thm}
\label{thm:minimum}
At time-step $t$, with probability $1-2\delta$ and number of i.i.d. trajectories being $N=\frac{2(1+\lambda_t \Cphi)^2}{\varepsilon^2(1-\gamma)^2}\log\frac{T_t}{\delta}$ for each policy $\pi_i$,
\begin{align}
    J_\nu(\pi_{\lambda_t}^*,\lambda_t)-J_\nu(\pi_{\widehat{i}},\lambda_t)\le \min_{i=0,...,T_t-1}J_\nu(\pi_{\lambda_t}^*,\lambda_t)- J_\nu(\pi_{i+1},\lambda_t)+2\varepsilon
\end{align}
where $\widehat{i}=\argmax_{i=0,...,T_t-1} \widehat{J}_\nu(\pi_{i+1},\lambda_t )$.
\end{thm}

\subsubsection{Policy Gradient Based $\alg$: Gradient Ascent with Softmax Parameterization}
\label{sec:softpg}
In Remark~\ref{rem:lsmooth}, it's mentioned that $\Omega(\pi)=\sum_a\pi(a|s)\log(\pi(a|s))$ breaks the $L$-smooth assumption in tabular MDP setting. However, if policy is parameterized by softmax, the $L$-smooth property holds. In this scenario, \citet{mei2020global} showed that the iteration complexity of policy gradient (Algorithm~\ref{alg:softpga}) is $O(\log\frac{1}{\varepsilon})$ for a fixed $\lambda$ (Theorem~\ref{thm:pgsoft}). However, in this section, we show that the iteration complexity is inefficient (exponentially dependent on $1/\varepsilon$) for general MDPs. In such situation, can \textsf{AdaptReduce} help make the algorithm efficient (polynominally dependent on $1/\varepsilon$)? We have a negative answer.
\begin{algorithm}[htb]
	\caption{Policy Gradient Ascent}
	\label{alg:softpga}
	\begin{algorithmic}
		\STATE {\bfseries Input:} an initial parameter $\theta_0$, a regularization parameter $\lambda$ and $T$ the number of iteration.
		\FOR{iteration $t=0$ {\bfseries to} $T-1$}
		\STATE $\theta_{t+1} = \theta_t + \eta \nabla_\theta J(\pi_{\theta_t}, \lambda)$
		\ENDFOR
		\STATE {\bfseries Return: $\pi_T$} 
	\end{algorithmic}
\end{algorithm}

\begin{defn}[Softmax Parameterization]
	\label{def:softmax}
	Given a vector $\theta\in\mathbb{R}^{\SM\times\AM}$, the softmax parameterization of policy $\pi$ is defined as:
	\begin{align}
		\pi_\theta(a|s)=\frac{e^{\theta(s,a)}}{\sum_{a'\in\AM}e^{\theta(s,a')}}\notag
	\end{align}
\end{defn}

\begin{thm}[Theorem 6 in \cite{mei2020global}]
	\label{thm:pgsoft}
	Suppose the initial distribution $\mu(s)>0$ for all $s\in\SM$ and policy $\pi$ is parameterized by Definition~\ref{def:softmax}, policy gradient with learning rate $\eta=\frac{(1-\gamma)^3}{8+\lambda(4+8\log|\AM|)}$ outputs a policy $\pi_t$ satisfying:
	\begin{align}
		J(\pi_{\lambda}^*,\lambda)-J(\pi_t, \lambda)\le \left\|\frac{1}{\mu}\right\|_\infty\frac{J(\pi_\lambda^*,\lambda)-J(\pi_0,\lambda)}{1-\gamma}e^{-C_\lambda t}
	\end{align}
	where $C_\lambda=\frac{c_\lambda(1-\gamma)^4\min_s \mu(s)}{(8/\lambda+4+8\log|\AM|)|\SM|}\left\|\frac{d_{\pi_\lambda^*, \mu}}{\mu}\right\|_\infty^{-1}$.
\end{thm}

\begin{rem}
    \label{rem:ineffi}
    If we consider $C_\lambda$ as a constant, the iteration could be improved upon $\widetilde{O}(\log\frac{1}{\varepsilon})$. However, it's not the whole story. For a fixed $\lambda$, in order to obtain an $\varepsilon$-optimal policy by Theorem~\ref{thm:pgsoft} in polynominal time, we need to investigate how large $C_\lambda^{-1}$ could be, which is equivalent to investigate how large $1/c_\lambda$ could be.
    Note that $c_{\lambda}=\inf_{t\ge1}\min_{s,a}\pi_t(a|s)\approx \min_{s,a}\pi_\infty(a|s)\approx\min_{s,a}\pi_\lambda^*(a|s)$ by \cite{mei2020global}. However, we have the close form of $\pi_\lambda^*$:
    \begin{align}
        \pi_\lambda^*(\cdot|s)=\frac{e^{Q_\lambda^*(s,\cdot)/\lambda}}{\sum_{a\in\AM}e^{Q_\lambda^*(s,a)/\lambda}}
    \end{align}
    Thus, for each state $s$ and assuming $Q_{\lambda}^*(s,a_1)\le...\le Q_{\lambda}^*(s,a_{|\AM|})$, we can lower and upper bound $\min_a \pi_\lambda^*(a|s)$ as follows:
    \begin{align}
        \frac{1}{|\AM|e^{\left(Q_{\lambda}^*(s,a_{|\AM|})-Q_{\lambda}^*(s,a_1)\right)/\lambda}}\le\min_a \pi_\lambda^*(a|s)\le\frac{1}{e^{\left(Q_{\lambda}^*(s,a_{|\AM|})-Q_{\lambda}^*(s,a_1)\right)/\lambda}}
    \end{align}
    Note that $\Delta_\lambda^*:\overset{\Delta}{=}Q_{\lambda}^*(s,a_{|\AM|})-Q_{\lambda}^*(s,a_1)\le\frac{1+\lambda\log|\AM|}{1-\gamma}$. 
    When $\Delta_\lambda^*\approx\frac{1+\lambda\log|\AM|}{1-\gamma}$, we obtain $\min_a \pi_\lambda^*(a|s)$ is of order $\Theta(e^{-\frac{1}{\lambda(1-\gamma)}}|\AM|^{-\frac{1}{1-\gamma}})$, which is even exponentially dependent on $\frac{1}{1-\gamma}$. 
    In order to obtain an $\varepsilon$-optimal policy by Theorem~\ref{thm:pgsoft} for a fixed $\lambda$ in polynominal time, we have to assume $\Delta_\lambda^*\approx \lambda\log\left(\poly(\frac{1}{\lambda}, \frac{1}{1-\gamma})\right)$. How to reduce the exponential dependence on other term still remains open. We consider it as future work.
\end{rem}

By setting $\lambda=O(\varepsilon(1-\gamma)/\Cphi)$, the iteration complexity for Algorithm~\ref{alg:softpga} is $\widetilde{O}\left(e^{\frac{1}{\varepsilon(1-\gamma)}}\right)$, as $\Delta_\lambda^*\rightarrow\Delta^*>0$ while $\lambda\rightarrow 0$. Thus, the upper bound is not efficient. Though, we can also obtain that policy gradient method satisfies $\propi$ property in Corollary~\ref{cor:softahcc}. Negatively, we can't upper bound total time $\sum_{t=0}^{T-1}\tm(\lambda_t)$ polynominally dependent on $\frac{1}{\varepsilon}$ and $\frac{1}{1-\gamma}$ as the sub-solver is inefficient when $\lambda$ is sufficiently small. In fact, the total time is of order:
\begin{align}
    \sum_{t=0}^{T-1}\frac{1}{\lambda_t c_{\lambda_t}}\approx\sum_{t=0}^{T-1}\frac{1}{\lambda_t}e^{\frac{1}{\lambda_t}}
\end{align}
Note that $\lambda_T=O((1-\gamma)\varepsilon/\Cphi)$, thus no matter how to adjust the decaying rate of $\lambda_t$ we always obtain $e^{1/\varepsilon(1-\gamma)}$ in total time. In conclusion, we argue that efficient sub-solver algorithm is a sufficient condition for \textsf{AdaptReduce} leveraging an efficient algorithm. Besides, it's not saying that vanilla policy gradient method with softmax parameterization is inefficient at all for regularized MDP. We leave it as future work to derive an efficient upper bound in this situation.
\begin{cor}[Policy Gradient satisfies Prop-I(0)]
    \label{cor:softahcc}
	Under the same setting in Theorem~\ref{thm:pgsoft}, Policy Gradient satisfies the $\propi$ property with $\hat{\varepsilon}=0$ and:
	\begin{align}
		\tm(\lambda)=\frac{1}{C_\lambda}\log\left(\frac{4\left\|\frac{1}{\mu}\right\|_\infty}{1-\gamma}\right)
	\end{align}
\end{cor}

\subsubsection{Other Algorithms} 
\label{sec:dissother}
\citet{shani2019adaptive} proposed another type of policy gradient to solve regularized MDP as follows:
\begin{align}
    \pi_{t+1}=\argmin_{\pi\in\Delta(\AM)^{\SM}} -\langle\nabla J(\pi_t,\lambda),\pi-\pi_t\rangle+\frac{1}{\eta}\mathbb{E}_{s\sim d^{\pi_t}}D_{\Omega}(\pi(\cdot|s)||\pi_t(\cdot|s))
    \label{eq:trpo}
\end{align}
When $\Omega(\pi(\cdot|s))=\sum_{a}\pi(a|s)\log(\pi(a|s))$, we can obtain a close form of update from equation~(\ref{eq:trpo}):
\begin{align}
    \pi_{t+1}(\cdot|s)\propto \pi_t^{1-\lambda\eta}(\cdot|s)\odot e^{\eta Q_\lambda^{\pi_t}(s,\cdot)}
\end{align}

\noindent However, by definition of $\nabla J(\pi,\lambda)$, we can also re-write the update rule~(\ref{eq:trpo}) into:
\begin{align}
    \pi_{t+1}=\argmin_{\pi\in\Delta(\AM)^{\SM}}\mathbb{E}_{s\sim d_{\pi_t}}\left(-\langle Q_\lambda^{\pi_t}(s, \cdot)-\lambda\nabla \Omega(\pi_{t}(\cdot|s)), \pi(\cdot|s)-\pi_t(\cdot|s)\rangle+\frac{1}{\eta}D_\Omega(\pi(\cdot|s)||\pi_t(\cdot|s))\right)
\end{align}
As $D_\Omega(\pi||\pi_t)=\Omega(\pi)-\Omega(\pi_t)-\langle\nabla\Omega(\pi_t), \pi-\pi_t\rangle$ and we assume $\lambda\eta=1$, the update rule can be simplified to as:
\begin{align}
    \pi_{t+1}=\argmin_{\pi\in\Delta(\AM)^{\SM}}\mathbb{E}_{d_{\pi_t}}-\langle Q_\lambda^{\pi_t}(s, \cdot)-\lambda \Omega(\pi(\cdot|s)), \pi(\cdot|s)\rangle\label{eq:trpo_final}
\end{align}
which is exactly RMPI with exact evaluation version ($m=\infty$) by solving above problem for each state $s\in\SM$. \citet{cen2020fast} also showed near the same update rule as equation~(\ref{eq:trpo_final}) while $\pi$ is parameterized by softmax and optimized by natural gradient. Thus, the analysis of convergence rate with $\lambda$ reduction can be covered by Section~\ref{sec:egrmpi}.
\begin{rem}
    In fact, we can also set $\eta=\Theta(1)$ and controlling $\lambda$ in \cite{shani2019adaptive, cen2020fast}. Thus the iteration complexity for fixed small $\lambda=O(\varepsilon(1-\gamma)/\Cphi)$ is $O\left(\frac{1}{\varepsilon(1-\gamma)}\log\frac{1}{\varepsilon}\right)$. Besides, combining with Algorithm~\ref{alg:ar}, the final iteration complexity is also $O\left(\frac{1}{\varepsilon(1-\gamma)}\right)$ (a logarithm term can be removed) in terms of $\varepsilon$. Either methods is slower than the case $\lambda\eta=1$. 
\end{rem}

\section{On discussion with \textsf{AdaptReduce} and Primal-Dual}
\label{sec:minmax}
In this section, we illustrate the motivation of Algorithm~\ref{alg:ar} by primal dual perspective. In fact, solving unregularized MDP equals with solving regularized MDP by decaying $\lambda$ as Theorem~\ref{thm:maxmax} describes.

\begin{thm}[Strong Duality]
	\label{thm:maxmax}	
	Under Assumption~\ref{ass:reg}, $ | \max_{\pi}  J(\pi, \lambda)| < \infty$ for any fixed $\lambda$.
	Then
	\begin{equation}
	\label{eq:maxmax}
	J(\pi^*) =
	 \max_{\pi} \min_{\lambda \ge 0} J(\pi, \lambda) 
	=  \min_{\lambda \ge 0}  \max_{\pi} J(\pi, \lambda),
	\end{equation}
	where $J(\pi, \lambda)$ is defined in~\eqref{eq:goal}, $J(\pi) := J(\pi, 0)$ and $\pi^* = \argmax_{\pi} J(\pi)$ is the optimal policy.
\end{thm}

The non-positivity of $\Omega$ and uniformly bounded property are crucial for the establishment of  Theorem~\ref{thm:maxmax}, which makes sure that the optimal $\lambda^*$ solving the minimax problem will be exactly zero.
In this case, the corresponding optimal policy will be our target $\pi^*$.
Theorem~\ref{thm:maxmax} bridges regularized MDPs and the optimal policy $\pi^*$.
It indicates that we could make use of the exchangeability of maximum and minimum to design efficient algorithms.
It also enriches our tools to analyze the effect of $\lambda$ decay.    

The adaptive reduction algorithm introduced in Section~\ref{sec:reduction}
decays $\lambda$ exponentially fast and require the sub-solver algorithm satisfies the $\propi$ or $\propii$ property.
Based on the discussion of Section~\ref{sec:example}, intuitively, as long as we run a proper solver algorithm for a enough long time, $\propi$ or $\propii$ property can be met, which implies between two consecutive lambda decay happens many steps of policy gradient descent. However, it's not saying that $\lambda$ decaying exponentially is the only scheme. Compare with other decaying scheme, such as $\lambda_t=\frac{1}{t^\alpha}$ or $\lambda_t=\frac{1}{\log(t+1)}$, $\lambda_t=\frac{1}{2^t}$ is the best we can find for sub-solvers we have analyzed. We have shown that the iteration complexity for obtaining the $\varepsilon$-optimal policy is $\widetilde{O}\left(\frac{1}{\varepsilon}\right)$ in terms of first order gradient method without parameterization. We still don't know whether there exists a better decaying scheme to obtain a better convergence result. But we conjecture that the best we can hope is $\widetilde{O}\left(\frac{1}{\varepsilon}\right)$ in terms of $\varepsilon$ (Conjecture~\ref{conj:1}). For example, if we let projected gradient ascent satisfies $\alg(\pi_t,\lambda_t,\varepsilon_t,0)$ with $\varepsilon_t=\lambda_t\frac{\Cphi}{1-\gamma}$, we have the following proposition holds, where we only focus on order of $\lambda_t$ (let $\lambda_{-1}=\lambda_0$):
\begin{align}
    D_t\le2\lambda_{t-1}\frac{\Cphi}{1-\gamma}\\
    \tm(\lambda_t)=\Theta(\frac{D_t}{\lambda_t^2})
\end{align}
Thus, the total time could be $\sum_{t=0}^{T-1}\tm(\lambda_t)=\Theta(\sum_{t=0}^{T-1}\frac{\lambda_{t-1}}{\lambda_t^2})$. If $\lambda_t/\lambda_{t-1}=o(1)$, which is faster than $\lambda_t=1/2^t$, then the total time could be larger than $1/\varepsilon$ by the fact $\lambda_T=O(\varepsilon)$.

 Though, \citet{bhandari2020note} conjectures that projected gradient ascent for tabular unregularized MDP can achieve linear convergence rate, we reckon linear convergence rate can only be met with more information (e.g. $d_{\pi}(s)$) than just gradient. For other terms such as $\frac{1}{1-\gamma}$, we argue that it's still loose. Based on Theorem~\ref{thm:maxmax}, how to design a more efficient way to reduce the dependence with $\frac{1}{1-\gamma}$, we leave it as future work. 
Note that current works (including ours) on regularized MDP are based on optimization perspective, it still remains open how regularization term effects exploration in RL from theoretical perspective.


\begin{conj}[Lower bound]
\label{conj:1}
There exist MDPs, given learning rate $\eta_t$ satisfies $\eta_t\le\eta$, where $\eta$ is a constant,  and oracle:
\begin{align}
    \pi_{t+1}\in \argmin_{\pi\in\Delta(\AM)^{\SM}}-\langle\nabla J(\pi_t, \lambda_t),\pi-\pi_t\rangle+\frac{1}{2\eta_t}\|\pi-\pi_t\|_{2}^{2}
\end{align}
For any $\{\lambda_t\}_{t=1}^{T}$ such that $\lambda_{t+1}\le\lambda_t$ and $\lim_{t\rightarrow\infty}\lambda_t=0$, we have $V^{\pi^*}-V^{\pi_{T}}=\widetilde{\Omega} (\frac{1}{T})$, where we ignore parameters polynominally dependent on MDP.
\end{conj}

\begin{rem}
In Conjecture~\ref{conj:1}, we argue that $\eta_t$ is bounded is important. If we let $\eta_t$ could be arbitrary positive, we can always do exact line search such that the iteration complexity could be $O(\log\frac{1}{\varepsilon})$ \cite{bhandari2020note}.
\end{rem}

\section{Conclusion}
In this paper, we propose an adaptive reduction method to decay coefficient of regularization term. It's shown that the iteration complexity for obtaining an $\varepsilon$-optimal policy can be accelerated by our reduction method compared with setting $\lambda$ sufficiently small. However, when sub-solver is inefficient, our method may also become inefficient. Besides, we also connect our method with primal dual problem. We argue that reduction scheme can be regarded as solving dual problem. Thus, how to design a more efficient decaying scheme for some other sub-solvers in regularized RL is left as future work.

\bibliography{bib/refer}
\bibliographystyle{plainnat}
\appendix
\begin{appendix}
	\onecolumn
	\begin{center}
		{\huge \textbf{Appendix}}
	\end{center}

\section{Proof of Section~\ref{sec:reduction}}

\subsection{Proof of Theorem~\ref{thm:reduction}}

\begin{lem}[Boundness of $\Phi^{\pi}$]
	\label{lem:bound_phi}
	For any regularization function $\Omega$ satisfying Assumption~\ref{ass:reg}, it follows that for any policy $\pi$,
	\[  \| \Phi^{\pi}\|_{\infty} \le \frac{\Cphi}{1-\gamma}
	\  \text{and} \
	|\Phi^{\pi}(\mu) | \le \frac{\Cphi}{1-\gamma}.
	   \]
\end{lem}
\begin{proof}
	It is easy to verify that for any given $\pi$ and for any $s \in \SM$,
	\[  \Phi^{\pi}(s) 
	=  \EB_{\tau \sim \pi} \left[  \sum_{t=0}^{\infty}\gamma^t
	\Omega(\pi, s_t)|s_0 = s \right]
	=  \frac{1}{1-\gamma} \sum_{s' \in \SM} d_s^{\pi}(s') \Phi(\pi, s')   \]
	which directly implies that $ |\Phi^{\pi}(s) | \le \frac{\Cphi}{1-\gamma}$ uniformly in $s$.
\end{proof}

With Lemma~\ref{lem:bound_phi}, we can prove that Algorithm~\ref{alg:ar} with any $\alg$ satisfying $\propi$ will find the optimal policy in an exponentially fast speed.

\begin{proof}[Proof of Theorem~\ref{thm:reduction}]
	For any given $s \in \SM$, define $D_t(s) = V_{\lambda_t}^*(s) - V_{\lambda_t}^{\hat{\pi}_t}(s)$ to be the initial value difference between $\hat{\pi}_t$ and the optimal policy of $\mathcal{M}(\lambda_{t})$ before we call $\alg$ in iteration $t$.
	Let $D_t = \Emu D_t(s_0) = V_{\lambda_t}^*(\mu) - V_{\lambda_t}^{\hat{\pi}_t}(\mu)$ for simplicity (which is always non-negative).
	Then,
	\begin{align*}
	0 &\le  D_t = 
	V_{\lambda_t}^*(\mu) - V_{\lambda_t}^{\hat{\pi}_{t}}(\mu)
	\overset{(a)}{=}  V_{\lambda_t}^{\pi_{\lambda_t}^*}(\mu) - V_{\lambda_t}^{\hat{\pi}_{t}}(\mu) \\
	&\overset{(b)}{=} V_{\lambda_{t-1}}^{\pi_{\lambda_{t}}^*}(\mu) - V_{\lambda_{t-1}}^{\hat{\pi}_{t}}(\mu) + (\lambda_{t-1} - \lambda_t) \Phi^{\pi_{\lambda}^*}(\mu) -  (\lambda_{t-1} - \lambda_t) \Phi^{\hat{\pi}_{t}}(\mu) \\
	&\overset{(c)}{\le}  V_{\lambda_{t-1}}^{\pi_{\lambda_{t}}^*}(\mu) - V_{\lambda_{t-1}}^{\hat{\pi}_{t}}(\mu) + 2(\lambda_{t-1} - \lambda_t) \frac{\Cphi}{1-\gamma} \\
	&\overset{(d)}{\le}  V_{\lambda_{t-1}}^{*}(\mu) - V_{\lambda_{t-1}}^{\hat{\pi}_{t}}(\mu) + 2(\lambda_{t-1} - \lambda_t) \frac{\Cphi}{1-\gamma}\\
	&\overset{(e)}{=}   V_{\lambda_{t-1}}^{*}(\mu) - V_{\lambda_{t-1}}^{\hat{\pi}_{t}}(\mu)  +  \frac{\lambda_0}{2^{t-1}} \frac{\Cphi}{1-\gamma}\\
	&\overset{(f)}{\le}  \frac{1}{4} D_{t-1} + \hat{\varepsilon} +  \frac{\lambda_0}{2^{t-1}} \frac{\Cphi}{1-\gamma}
	\end{align*}
	where (a) follows from the notation of the optimal policy $V_{\lambda_t}^* = V_{\lambda_{t}}^{\pi_{\lambda_{t}}^*}$; (b) uses the value function decomposition~\eqref{eq:decom}; (c) uses the boundness of $\Omega$ (which is $\|\Phi^{\pi}\|_{\infty} \le \frac{\Cphi}{1-\gamma}$); (d) uses the the optimality of $\pi_{\lambda_{t-1}}^*$, i.e., $ V_{\lambda_{t-1}}^{\pi_{\lambda_{t}}^*}(\mu) \le V_{\lambda_{t-1}}^{*}(\mu)  := V_{\lambda_{t-1}}^{\pi_{\lambda_{t-1}}^*}(\mu)$;
	(e) uses the fact that $\lambda_{t} = \frac{\lambda_0}{2^t}$; and 
	(f) uses the $\propi$ property of $\alg$, which ensures that $  V_{\lambda_{t-1}}^{*}(\mu) - V_{\lambda_{t-1}}^{\hat{\pi}_{t}}(\mu)
	\le  \frac{1}{4} ( V_{\lambda_{t-1}}^{*}(\mu) - V_{\lambda_{t-1}}^{\hat{\pi}_{t-1}}(\mu))  + \hat{\varepsilon} 
	= \frac{1}{4} D_{t-1} + \hat{\varepsilon}$.
	Recursively applying the above inequality, we have
	\begin{equation}
	\label{eq:D_T}
	D_T \le \frac{D_0}{4^T} +  \frac{4}{3}\hat{\varepsilon} +   \frac{\lambda_0\Cphi}{1-\gamma}  \sum_{i=0}^{T-1} \frac{1}{2^{T-1-i}4^i}  
	\le \frac{D_0}{4^T} +  \frac{4}{3} \hat{\varepsilon} +   \frac{4\lambda_0\Cphi}{1-\gamma}  \frac{1}{2^T} .
	\end{equation}
	
	In sum, we obtain a policy $\hat{\pi}_T$ satisfying
	\begin{align*}
	0 
	&\le V^{*}(\mu) - V^{\hat{\pi}_{T}}(\mu) 
	\overset{(a)}{=} V^{\pi^*}(\mu) - V^{\hat{\pi}_{T}}(\mu)\\
	&\overset{(b)}{=} V_{\lambda_{T}}^{\pi^*}(\mu) - V_{\lambda_{T}}^{\hat{\pi}_{T}}(\mu) + \lambda_{T}\Phi^{\hat{\pi}_T}(\mu) -\lambda_{T} \Phi^{\pi^*}(s)\\
	&
	\overset{(c)}{\le} V_{\lambda_{T}}^{\pi^*}(\mu) - V_{\lambda_{T}}^{\hat{\pi}_{T}}(\mu) + 2\lambda_{T} \frac{\Cphi}{1-\gamma} \\
	&	\overset{(d)}{\le} V_{\lambda_{T}}^{*}(\mu) - V_{\lambda_{T}}^{\hat{\pi}_{T}}(\mu) + 2\lambda_{T} \frac{\Cphi}{1-\gamma} 
	\end{align*}
	where (a) follows from the notation of the optimal policy $V^{*} = V^{\pi^*}$; (b) uses the value decomposition of $V_{\lambda_T}^{*}$ and $V_{\lambda_T}^{\hat{\pi}_{T}}$ (defined in~\eqref{eq:decom}); (c) uses the boundness of $\Phi^{\pi}$ shown in Lemma~\ref{lem:bound_phi}; and (d) uses the the optimality of $\pi_{\lambda_{T}}^*$, i.e., $ V_{\lambda_{T}}^{\pi^*}(\mu) \le V_{\lambda_{T}}^{*}(\mu)  := V_{\lambda_{T}}^{\pi_{\lambda_T}^*}(\mu)$.
	
	Since the bound holds uniformly for all $s \in \SM$, we have
	\begin{align*}
		V^{*}(\mu) - V^{\hat{\pi}_{T}}(\mu) 
	&\le V_{\lambda_{T}}^{*}(\mu)- V_{\lambda_{T}}^{\hat{\pi}_{T}}(\mu)
	+ 2\lambda_{T} \frac{\Cphi}{1-\gamma} \\
	&\overset{(a)}{=} D_T + \frac{2\lambda_0\Cphi}{1-\gamma}  \frac{1}{2^T} \overset{(b)}{\le} \frac{D_0}{4^T} + \frac{4}{3} \hat{\varepsilon} +  \frac{6\lambda_0\Cphi}{1-\gamma}  \frac{1}{2^T}.
	\end{align*}

	Above, (a) uses the definition of $D_T$ and $\lambda_{T}$ and (b) applies~\eqref{eq:D_T}.
\end{proof}

\subsubsection{Proof of Theorem~\ref{thm:reduction2}}
\begin{proof}
    Similar with Theorem~\ref{thm:reduction}, we denote $D_t(s)=V^*_{\lambda_t}(s)-V^{\hpi_t}_{\lambda_t}(s)$ to be the initial value difference and $D_t=V^*_{\lambda_t}(\mu)-V^{\hpi_t}_{\lambda_t}(\mu)$. By proof of Theorem~\ref{thm:reduction},
    \begin{align}
        0\le D_t&\le V^*_{\lambda_{t-1}}(\mu)-V^{\hpi_t}_{\lambda_{t-1}}(\mu)+\frac{\lambda_0}{2^{t-1}}\frac{C_\phi}{1-\gamma}\notag\\
        &\overset{(a)}{\le}\frac{4\lambda_0}{2^t}\frac{C_\phi}{1-\gamma}
        \label{eq:dt}+\hat{\varepsilon}
    \end{align}
    where (a) follows from $\propii$.
    
    Thus, for the output $\hpi_T$, we have:
    \begin{align}
        0\le V^*(\mu)-V^{\hpi_T}(\mu)&\le D_T+\frac{2\lambda_0}{2^T}\frac{C_\phi}{1-\gamma}\notag\\
        &\le \frac{6\lambda_0}{2^T}\frac{C_\phi}{1-\gamma}+\hat{\varepsilon}
    \end{align}
\end{proof}

\subsection{Proof related to RMPI}
\subsubsection{Proof of Theorem~\ref{thm:RMPI}}
\begin{proof}
	Here we will make use of Corollary 1 in~\cite{geist2019theory} or Theorem 7 in~\cite{scherrer2015approximate} to prove Theorem~\ref{thm:RMPI}. 
	Let $\nu, \mu$ be two state distributions with $\nu$ used for initial state distribution and $\mu$ for performance measure.
	Let $p, q$ and $q'$ such that $\frac{1}{q}+\frac{1}{q^{\prime}}=1$.
	We define a norm by $\|V\|_{p, \mu} := \left[\EB_{s_0 \sim \mu} |V(s_0)|^p\right]^{\frac{1}{p}}$ and the concentrability coefficients as:
	\[
	C_{q}^{i}=\frac{1-\gamma}{\gamma^{i}} \sum_{j=i}^{\infty} \gamma^{j} \max _{\pi_{1}, \ldots, \pi_{j}}\left\|\frac{\nu P_{\pi_{1}} P_{\pi_{2}} \ldots P_{\pi_{j}}}{\mu}\right\|_{q, \mu}.
	\]
	Then, by Corollary 1 in~\cite{geist2019theory}, we have
	\begin{equation}
	\label{eq:e1}
		\left\| V_{\lambda}^* -V_\lambda^{\alpha_k}  \right\|_{p, \rho} \leq 2 \sum_{i=1}^{k-1} \frac{\gamma^{i}}{1-\gamma}\left(C_{q}^{i}\right)^{\frac{1}{p}}\left\|\epsilon_{k-i}\right\|_{p q^{\prime}, \mu}+\sum_{i=0}^{k-1} \frac{\gamma^{i}}{1-\gamma}\left(C_{q}^{i}\right)^{\frac{1}{p}}\left\|\epsilon_{k-i}^{\prime}\right\|_{p q^{\prime}, \mu}+g(k)
	\end{equation}
	with 
	\[
	g(k)=\frac{2 \gamma^{k}}{1-\gamma}\left(C_{q}^{k}\right)^{\frac{1}{p}} \min \left(\left\| V_{\lambda}^* - V_\lambda^{\alpha_{0}}  \right\|_{p q^{\prime}, \mu},\left\|V_\lambda^{\alpha_{0}} -  \mathcal{T}_{\lambda}^{\alpha_1} V_\lambda^{\alpha_{0}}  \right\|_{p q^{\prime}, \mu}\right).
	\]
	Here we set $p=q'=1$, which implies $q = \infty$, and $\epsilon_k = \epsilon_k' = \epsilon_0 1_S$ for $k \ge 0$.
	Under the choice of parameters, we can simplify $C_q^i$ by noting that 
	\begin{equation}
	\label{eq:C_q^i}
	\left\|\frac{\rho P_{\pi_{1}} P_{\pi_{2}} \ldots P_{\pi_{j}}}{\mu}\right\|_{\infty, \mu}  \le
	\left\|\frac{\rho P_{\pi_{1}} P_{\pi_{2}} \ldots P_{\pi_{j}}}{\mu}\right\|_{\infty} \le
	\frac{\max_{s \in \SM} \nu(s)}{\min_{s \in \SM} \mu(s)}, \
	\text{for any} \ j \ge 1 \ \text{and} \ \pi_1, \cdots, \pi_j.
	\end{equation}
	Therefore, $C_{\infty}^i \le \frac{\max_{s \in \SM} \nu(s)}{\min_{s \in \SM} \mu(s)}$ for all $i \ge 0$.

	If follows that
	\begin{align*}
	V_{\lambda}^*(\mu) -V_\lambda^{\alpha_k} (\mu)
	&=  | \EB_{s_0 \sim \mu } \left[V_{\lambda}^*(s_0) -V_\lambda^{\alpha_k} (s_0)  \right] |\\
	&\overset{(a)}{=}  \EB_{s_0 \sim \mu }| V_{\lambda}^*(s_0) -V_\lambda^{\alpha_k} (s_0)  | \\
	&= 	\left\| V_{\lambda}^* -V_\lambda^{\alpha_k}  \right\|_{1, \mu}\\
	&\overset{(b)}{\le} g(k) + \frac{2(\gamma - \gamma^k) \epsilon_0 }{(1-\gamma)^2}\max_{0\le i \le k} C_{\infty}^i  + \frac{(1-\gamma^k)\epsilon_0}{(1-\gamma)^2}\max_{0\le i \le k} C_{\infty}^i \\
	&\overset{(c)}{\le} \frac{2C\gamma^k}{1-\gamma} \left( V_{\lambda}^*(\mu) -V_\lambda^{\alpha_0} (\mu) \right) + \frac{(1+2\gamma)C}{(1-\gamma)^2}\epsilon_0
	\end{align*}
	where (a) follows the fact that $V_{\lambda}^*(s) \ge V_{\lambda}^{\pi}(s)$ for any $\pi$ and $s \in \SM$ satisfying $\mu(s) > 0$; (b) follows from~\eqref{eq:e1}; (c) uses the same argument for (a) that is $	V_{\lambda}^*(\mu) -V_\lambda^{\alpha_k} (\mu) = \left\| V_{\lambda}^* -V_\lambda^{\alpha_k}  \right\|_{1, \mu}$ and the notation $C = \max_{i \ge 0} C_{\infty}^i $.
	Then when the $k$ is large enough such that $ \frac{2C\gamma^k}{1-\gamma} \le \frac{1}{4}$, which implies $k \ge  \ln\frac{1-\gamma}{8C}/ \ln\frac{1}{\gamma}$, RMPI satisfy Prop-I$(\frac{(1+2\gamma)C}{(1-\gamma)^2}\epsilon_0)$ in times $\ln\frac{1-\gamma}{8C}/ \ln\frac{1}{\gamma}$.
\end{proof}
\subsubsection{Proof of Theorem~\ref{thm:rmpi-r}}
In this part, we denote $V_{T+1}=\TM_{\lambda_T}^{\pi_{T}}V_{T}+\varepsilon_0 1_{S}$.
\begin{lem}
	\label{lem:rmpi-1}
	$\|V_{T+1}-V_{\lambda_T}^{\pi_T}\|_{\infty}\le\frac{\gamma}{1-\gamma}\|V_T-V_{T+1}\|_{\infty}+\frac{1}{1-\gamma}\varepsilon_0$
\end{lem}
\begin{proof}
	By definition of Bellman equation, we have:
	\begin{align}
		V_{T+1}-V_{\lambda_T}^{\pi_T}&=\TM_{\lambda_{T}}^{\pi_T}V_T-V_{\lambda_T}^{\pi_T}+\varepsilon_0 1_S\\
		&=\TM_{\lambda_{T}}^{\pi_T}V_T-\TM_{\lambda_T}^{\pi_T}(\TM_{\lambda_T}^{\pi_T})^{\infty}V_T+\varepsilon_0 1_S
	\end{align}
	Taking infinite norm, we have:
	\begin{align}
		\|V_{T+1}-V_{\lambda_T}^{\pi_T}\|_{\infty}&\le\gamma\|V_T-(\TM_{\lambda_T}^{\pi_T})^{\infty}V_T\|_{\infty}+\varepsilon_0\\
		&\le\gamma\sum_{n=0}^{\infty}\|(\TM_{\lambda_T}^{\pi_T})^{n}V_T - (\TM_{\lambda_T}^{\pi_T})^{n+1}V_T\|_{\infty}+\varepsilon_0\\
		&\le\gamma\sum_{n=0}^{\infty}\gamma^{n}\|V_T-\TM_{\lambda_T}^{\pi_T}V_T\|_{\infty}+\varepsilon_0\\
		&\le\frac{\gamma}{1-\gamma}\|V_T-V_{T+1}\|_{\infty}+\frac{1}{1-\gamma}\varepsilon_0
	\end{align}
\end{proof}

\begin{lem}
	\label{lem:rmpi-2}
	For any $t=0,1,...,T$, $\|V^*-V_{t+1}\|_{\infty}\le\lambda_t C_\phi +\gamma \|V^{*}-V_t\|_{\infty}+\varepsilon_0$
\end{lem}

\begin{proof}
	By definition, we have:
	\begin{align}
		\TM V^* -\TM_{\lambda_t}V_t - \varepsilon_0 1_S\le V^*-V_{t+1}=\TM V^* - \TM_{\lambda_t}^{\pi_t}V_t -\varepsilon_0 1_S\le \TM V^* - \TM_{\lambda_t}V_t
	\end{align}
	Taking infinity norm both sides, we have:
	\begin{align}
		\|V^*-V_{t+1}\|_{\infty}&\le\|\TM V^*-\TM_{\lambda_t}V_t\|_{\infty}+\varepsilon_0\\
		&\le \max_{\pi}\lambda_t\|\langle\phi(\pi),\pi\rangle\|_{\infty}+\gamma\|P^{\pi}(V^*-V_t)\|_{\infty}+\varepsilon_0\\
		&\le\lambda_t C_\phi +\gamma \|V^{*}-V_t\|_{\infty}+\varepsilon_0
	\end{align}
\end{proof}

\begin{proof}[Proof of Theorem~\ref{thm:rmpi-r}]
	Denote $V_{T+1}=\TM_{\lambda_T}^{\pi_{T}}V_{T}+\varepsilon_0 1_{S}$, we have:
	\begin{align}
		0\le V^*-V^{\pi_T}&= V^*- V^{\pi_T}_{\lambda_T}-\lambda_T\Phi(\pi_T)\\
		&\le V^*- V^{\pi_T}_{\lambda_T}+\frac{\lambda_T}{1-\gamma}C_\phi\\
		&= V^* - V_{T+1} + V_{T+1} - V_{\lambda_T}^{\pi_T}+\frac{\lambda_T}{1-\gamma}C_\phi
	\end{align}
	By Lemma~\ref{lem:rmpi-2} and assuming $\frac{1}{2}\le\gamma^2$, we obtain the following inequality by recursion and $\lambda_{t+1}=\frac{1}{2}\lambda_t$:
	\begin{align}
		\|V^*-V_{T}\|_{\infty}\le\frac{\gamma^{T-1}}{1-\gamma}\lambda_0 C_\phi + \frac{1}{1-\gamma}\varepsilon_0+\gamma^{T}\|V^*-V_0\|_{\infty}
	\end{align}
	By Lemma~\ref{lem:rmpi-1}, we obtain the final bound of $\|V^*-V^{\pi_T}\|_{\infty}$:
	\begin{align}
		\|V^*-V^{\pi_T}\|_{\infty}&\le\|V^*-V_{T+1}\|_{\infty}+\|V_{T+1}-V_{\lambda_T}^{\pi_T}\|_{\infty}+\frac{\lambda_T}{1-\gamma}C_{\phi}\\
		&\le\|V^*-V_{T+1}\|_{\infty}+\frac{\gamma}{1-\gamma}\|V^*-V_{T+1}\|_{\infty}+\frac{\gamma}{1-\gamma}\|V^*-V_{T}\|_{\infty}+\frac{\varepsilon_0}{1-\gamma}+\frac{\lambda_T}{1-\gamma}C_{\phi}\\
		&=\frac{1}{1-\gamma}\|V^*-V_{T+1}\|_{\infty}+\frac{\gamma}{1-\gamma}\|V^*-V_{T}\|_{\infty}+\frac{\varepsilon_0}{1-\gamma}+\frac{\lambda_T}{1-\gamma}C_{\phi}\\
		&\le\frac{2\gamma}{1-\gamma}\|V^*-V_{T}\|_{\infty}+\frac{2\varepsilon_0}{1-\gamma}+\frac{2\lambda_T}{1-\gamma}C_{\phi}\\
		&\le\frac{2\gamma^T}{(1-\gamma)^2}\lambda_0 C_\phi+\frac{2}{(1-\gamma)^2}\varepsilon_0+\frac{2\gamma^{T+1}}{1-\gamma}\|V^*-V_0\|_{\infty}+\frac{2\lambda_T}{1-\gamma}C_\phi\\
		&\le\frac{4\gamma^T}{(1-\gamma)^2}\lambda_0 C_\phi+\frac{2}{(1-\gamma)^2}\varepsilon_0+\frac{2\gamma^{T+1}}{1-\gamma}\|V^*-V_0\|_{\infty}
	\end{align}
	Let $\varepsilon_0\le\frac{(1-\gamma)^2}{6}$ and $T\ge\max\{\frac{\log\frac{\varepsilon(1-\gamma)^2}{12\lambda_0 C_\phi}}{\log\gamma}, \frac{\log\frac{\varepsilon(1-\gamma)}{6\|V^*-V_0\|_{\infty}}}{\log\gamma}-1\}$, we have $\|V^*-V^{\pi_T}\|_{\infty}\le\varepsilon$
\end{proof}

\subsection{Proof related to Projected Gradient Ascent}
\subsection{Derivative of $J(\pi,\lambda)$ w.r.t. $\pi$}
\begin{lem}
    For any $(s,a)\in\SM\times\AM$,
    \begin{align}
        \frac{\partial J(\pi,\lambda)}{\pi(a|s)}=\frac{1}{1-\gamma}d_{\pi,\rho}(s)\left(Q_\lambda^{\pi}(s,a)-\lambda\nabla_{s,a}\Omega(\pi(\cdot|s))\right)
    \end{align}
    where $\rho$ is the initial distribution.
\end{lem}
\begin{proof}
    For given $s,a$, we have:
        \begin{align}
        \frac{\partial V^\pi_\lambda(s)}{\partial \pi(a|s)}&=\frac{\partial}{\partial \pi(a|s)}\sum_{a'}\pi(a'|s)Q_\lambda^{\pi}(s, a') - \lambda\Omega(\pi(\cdot|s))\\
        &=Q_\lambda^{\pi}(s, a)-\lambda\nabla_{s,a}\Omega(\pi(\cdot|s))+\sum_{a'}\pi(a'|s)\frac{\partial Q^{\pi}_\lambda(s, a')}{\partial \pi(a|s)}\\
        &=Q_\lambda^{\pi}(s, a)-\lambda\nabla_{s,a}\Omega(\pi(\cdot|s))+\gamma\sum_{a', s'}\pi(a'|s)P(s'|s, a')\frac{\partial V^{\pi}_\lambda(s')}{\partial \pi(a|s)}\\
        \frac{\partial V^\pi_\lambda(\tilde{s})}{\partial \pi(a|s)}&=\gamma\sum_{a', s'}\pi(a'|\tilde{s})P(s'|\tilde{s}, a')\frac{\partial V^{\pi}_\lambda(s')}{\partial \pi(a|\tilde{s})}
    \end{align}
    By recursively expanding $\frac{\partial V^\pi_\lambda(s')}{\partial \pi(a|s)}$, we get:
    \begin{align}
        \frac{\partial J(\pi,\lambda)}{\pi(a|s)}=\frac{1}{1-\gamma}d_{\pi,\rho}(s)\left(Q_\lambda^{\pi}(s,a)-\lambda\nabla_{s,a}\Omega(\pi(\cdot|s))\right)
    \end{align}
\end{proof}

\subsubsection{Smoothness of $J(\pi,\lambda)$ w.r.t. $\pi$}
The proof technique in this section is the same with \cite{agarwal2019optimality}. 
The difference is that we additionally consider the smoothness of regularization rather than the value function alone.

Define $\pi_\alpha(\cdot|s) = \pi(\cdot|s)+\alpha u_{s}$, where $u_s$ is an arbitrary direction with unit $\ell_2$ norm. Our aim is to bound 
\begin{align}
    \max_{\|u_{s}\|_{2}=1,\forall s}\left| \left. \frac{d^2 V^{\pi{_\alpha}}_\lambda(s)}{d\alpha^2}\right|_{\alpha=0}\right|
\end{align}

To that end, we have:
\begin{align}
    \frac{d^2 V^{\pi{_\alpha}}_\lambda(s)}{d\alpha^2}&=\sum_{a}\frac{d^2\pi_{\alpha}(a|s)}{d\alpha^2}\left(Q_{\lambda}^{\pi_{\alpha}}(a|s)-\lambda\Omega(\pi_{\alpha}(\cdot|s))\right)+2\sum_a \frac{d\pi_{\alpha}(a|s)}{d\alpha}\frac{d\left(Q_{\lambda}^{\pi_{\alpha}}(a|s)-\lambda\Omega(\pi_{\alpha}(\cdot|s))\right)}{d\alpha} \nonumber \\
    &+\sum_a\pi_\alpha(a|s)\frac{d^2\left(Q_{\lambda}^{\pi_{\alpha}}(a|s)-\lambda\Omega(\pi_{\alpha}(\cdot|s))\right)}{d\alpha^2},
\end{align}
and we are going to bound the above three terms separately.
By our definition, the first term is equal to zero because $\frac{d^2\pi_{\alpha}}{d\alpha^2}=0$.

For the rest, note that $Q_{\lambda}^{\pi_{\alpha}}(a|s)-\lambda\Omega(\pi_{\alpha}(\cdot|s))=1_{(s,a)}^T M(\alpha)(R-\lambda\Omega_\alpha)$, where $M(\alpha)=(I-\gamma\P(\alpha))^{-1}$, $\Omega_\alpha(s,a)=\Omega(\pi_\alpha(\cdot|s))$ and $\P(\alpha)(s',a'|s,a)=\pi_\alpha  (a'|s')\P(s'|s,a)$.
What's more,
\begin{align}
    \frac{d\left(Q_{\lambda}^{\pi_{\alpha}}(a|s)-\lambda\Omega(\pi_{\alpha}(\cdot|s))\right)}{d\alpha}&=\gamma 1_{(s,a)}^{T} M(\alpha)\frac{d \P(\alpha)}{d\alpha}M(\alpha)(R-\lambda\Omega_{\alpha})-\lambda 1_{(s,a)}^{T} M(\alpha)\frac{d\Omega_{\alpha}}{d\alpha}\\
    \frac{d^2 \left(Q_{\lambda}^{\pi_{\alpha}}(a|s)-\lambda\phi(\pi_{\alpha}(\cdot|s))\right)}{d\alpha^2}&=2\gamma^2 1_{(s,a)}^{T}M(\alpha)\frac{d \P(\alpha)}{d\alpha}M(\alpha)\frac{d \P(\alpha)}{d\alpha}M(\alpha)(R-\lambda\Omega_{\alpha})\\&-2\lambda\gamma 1_{(s,a)}^{T}M(\alpha)\frac{d \P(\alpha)}{d\alpha}M(\alpha)\frac{d\Omega_{\alpha}}{d\alpha}-\lambda 1_{(s,a)}^{T}M(\alpha)\frac{d^2 \Omega_{\alpha}}{d\alpha^2}
\end{align}

By definition, we have:
\begin{align}
    \|M(\alpha)x\|_{\infty}&\le\frac{1}{1-\gamma}\|x\|_{\infty}\\
    \max_{\|u_s\|_{2}=1}\left\|\frac{d\P(\alpha)}{d\alpha}x\right\|_{\infty}&\le\max_{\|u_s\|_{2}=1}\max_{(s,a)}\left|\sum_{s',a'}u_{s'}(a')\P(s'|s,a)x_{s',a'}\right|\\&\le\max_{\|u_s\|_{2}=1}\max_{(s,a)}\sum_{s'}\P(s'|s,a)\sum_{a'}\left|u_{s'}(a')\right|\|x\|_{\infty}\\
    &\le\sqrt{|\AM|}\|x\|_{\infty},
\end{align}
where the third inequality used $\sum_{a'}|u_{s'}(a')|\le\sqrt{|\AM|\sum_{a'}u_{s'}^{2}(a')}=\sqrt{|\AM|}$.

By assumption, $r(s,a)\in[0,1]$, $\Omega_{\alpha}\in[0,\Cphi]$, $|\Omega_{\alpha}'|\in[0,\Cphi^{(1)}]$, $|\Omega_{\alpha}''|\in[0,\Cphi^{(2)}]$, then
\begin{align*}
    &\max_{\|u_s\|_{2}=1}\left|\frac{d(Q_{\lambda}^{\pi_{\alpha}}(a|s)-\lambda\Omega_{\alpha}(\pi(a|s)))}{d\alpha}\right|\\
    &\le\gamma\max_{\|u_s\|_{2}=1}\left\|M(\alpha)\frac{d\P(\alpha)}{d\alpha}M(\alpha)(R-\lambda\Omega_{\alpha})\right\|_{\infty}+\lambda\max_{\|u_s\|_{2}=1}\left\|M(\alpha)\frac{d\Omega_{\alpha}}{d\alpha}\right\|_{\infty}\\
    &\le\frac{\gamma\sqrt{|\AM|}(1+\lambda \Cphi)}{(1-\gamma)^2}+\frac{\lambda \Cphi^{(1)}}{1-\gamma}
\end{align*}

Hence, the second term can be bounded as:
\begin{align}
    2\left|\sum_a u_{a,s}\frac{d\left(Q_{\lambda}^{\pi_{\alpha}}(a|s)-\lambda\Omega(\pi_{\alpha}(a|s))\right)}{d\alpha}\right|&\le2\sqrt{|\AM|}\left[\frac{\gamma\sqrt{|\AM|}(1+\lambda \Cphi)}{(1-\gamma)^2}+\frac{\lambda \Cphi^{(1)}}{1-\gamma}\right]\\
    &=\frac{2\gamma|\AM|(1+\lambda \Cphi)}{(1-\gamma)^2}+\frac{2\lambda \sqrt{|\AM|}\Cphi^{(1)}}{(1-\gamma)}
\end{align}

Next, we consider the third term:
\begin{align}
    \max_{\|u_s\|_{2}=1}\left|\frac{d^2(Q_{\lambda}^{\pi_{\alpha}}(a|s)-\lambda\Omega_{\alpha}(\pi(\cdot|s)))}{d\alpha^2}\right|
    &\le 2\gamma^2 \max_{\|u_s\|_{2}=1}\left\|M(\alpha)\frac{d\P(\alpha)}{d\alpha}M(\alpha)\frac{d\P(\alpha)}{d\alpha}M(\alpha)(R-\lambda\Omega_\alpha)\right\|_{\infty}\\
    &+2\lambda\gamma\max_{\|u_s\|_{2}=1}\left\|M(\alpha)\frac{d\P(\alpha)}{d\alpha}M(\alpha)\frac{d\Omega_{\alpha}}{d\alpha}\right\|_{\infty}\\
    &+\lambda\max_{\|u_s\|_{2}=1}\left\|M(\alpha)\frac{d^2\Omega_{\alpha}}{d\alpha^2}\right\|_{\infty}\\
    &\le\frac{2\gamma^2|\AM|(1+\lambda \Cphi)}{(1-\gamma)^3}+\frac{2\lambda\gamma\sqrt{|\AM|} \Cphi^{(1)}}{(1-\gamma)^2}+\frac{\lambda \Cphi^{(2)}}{1-\gamma}
\end{align}

Finally, we have:
\begin{align}
    \max_{\|u_{s}\|_{2}=1}\left|\left.\frac{d^2 V_{\lambda}^{\pi_{\alpha}}(s)}{d\alpha^2}\right|_{\alpha=0}\right|&\le
    \frac{2\gamma|\AM|(1+\lambda \Cphi)}{(1-\gamma)^2}+\frac{2\lambda\sqrt{|\AM|} \Cphi^{(1)}}{(1-\gamma)}
    +
    \frac{2\gamma^2|\AM|(1+\lambda \Cphi)}{(1-\gamma)^3}+\frac{2\lambda\gamma\sqrt{|\AM|} \Cphi^{(1)}}{(1-\gamma)^2}+\frac{\lambda \Cphi^{(2)}}{1-\gamma}\\
    &\le\frac{2\gamma|\AM|(C_{1}+\lambda C_{2})}{(1-\gamma)^3}+\frac{2\lambda\gamma C_{2}^{(1)}}{(1-\gamma)^{2}}+\frac{2\lambda C_{2}^{(1)}}{1-\gamma}+\frac{\lambda C_{2}^{(2)}}{\sqrt{|\AM|}(1-\gamma)}\\
    &\le\frac{4\gamma|\AM|}{(1-\gamma)^3}+\lambda\cdot\frac{4\gamma|\AM|\Cphi+2(1-\gamma)\sqrt{|\AM|}\Cphi^{(1)}+(1-\gamma)^2\Cphi^{(2)}}{(1-\gamma)^3}
\end{align}
which says $V_{\lambda}^{\pi}$ is $L$-smooth w.r.t. $\pi$.

\subsubsection{Convergence of $J_\mu(\pi,\lambda)$ for a fixed $\lambda$}
\begin{lem}
    \label{lem:diff}
    For any given $\pi$, $\lambda$, denote $\pi^{*}_\lambda$ is the optimal policy for regularized MDP with $\lambda$, we have the following equation:
    \begin{align}
        J_\mu(\pi^*_\lambda,\lambda)-J_\mu(\pi,\lambda) = \frac{1}{1-\gamma}\E_{d_{\pi^*_\lambda,\mu}\pi^*_\lambda}[A^{\pi}_{\lambda}(s,a)-\lambda\Omega(\pi_\lambda^*(\cdot|s))]
    \end{align}
    where $A_{\lambda}^{\pi}(s,a)=Q_{\lambda}^{\pi}(s,a)-V^{\pi}_{\lambda}(s)$
\end{lem}
\begin{proof}
    \begin{align}
        J_\mu(\pi^*_\lambda,\lambda)-J_\mu(\pi,\lambda) &= \E_{\pi^*_\lambda}\sum_{t=0}^{\infty}\gamma^t (r(s_t,a_t)-\lambda\Omega(\pi_{\lambda}^*(\cdot|s_t)))-J_\mu(\pi,\lambda)\\
        & = \E_{\pi^*_\lambda}\sum_{t=0}^{\infty}\gamma^t \left[r(s_t,a_t)+V_{\lambda}^{\pi}(s_{t})-V_{\lambda}^{\pi}(s_{t})\right]-J_\mu(\pi,\lambda)-\lambda\Phi_\mu(\pi^*_\lambda)\\
        & = \E_{\pi^*_\lambda}\sum_{t=0}^{\infty}\gamma^t \left[r(s_t,a_t)+\gamma V_{\lambda}^{\pi}(s_{t+1})-V_{\lambda}^{\pi}(s_{t})\right]-\lambda\Phi_\mu(\pi^*_\lambda)\\
        & = \E_{\pi^*_\lambda}\sum_{t=0}^{\infty}\gamma^t A_{\lambda}^{\pi}(s_t,a_t)-\lambda\Phi_\mu(\pi^*_\lambda)\\
        & = \frac{1}{1-\gamma}\E_{d_{\pi^*_\lambda,\mu}\pi^*_\lambda}[A^{\pi}_{\lambda}(s,a)-\lambda\Omega(\pi_\lambda^*(\cdot|s))]
    \end{align}
\end{proof}

\begin{lem}
    \label{lem:increase}
    Denote $\pi_{t+1}=Proj(\pi_t+\eta_\pi\nabla_\pi J_\nu(\pi_t,\lambda))$, we have the following improvement guarantee:
    \begin{align}
        J_\nu(\pi_{t+1},\lambda)-J_\nu(\pi_t,\lambda)\ge\frac{2-\eta_\pi L}{2\eta_\pi}\|\pi_{t+1}-\pi_t\|_{2}^2,
    \end{align}
    where $L$ is the smoothness coefficient of $J_\nu(\pi,\lambda)$.
\end{lem}

\begin{proof}
    By smoothness of $J_\nu(\pi,\lambda)$, we have:
    \begin{align}
        J_\nu(\pi_{t+1},\lambda)\ge J_\nu(\pi_t,\lambda)+\langle\nabla_\pi J_\nu(\pi_t,\lambda),\pi_{t+1}-\pi_t\rangle-\frac{L}{2}\|\pi_{t+1}-\pi_t\|^{2}_{2}
    \end{align}
    By first order stationary condition, we have:
    \begin{align}
        \langle\pi_{t+1}-\pi_t-\eta_\pi\nabla_\pi J_\nu(\pi_t,\lambda),\pi_{t+1}-\pi_t\rangle\le 0
    \end{align}
    So we obtain the final result:
    \begin{align}
    J_\nu(\pi_{t+1},\lambda)-J_\nu(\pi_t,\lambda)\ge(\frac{1}{\eta_\pi}-\frac{L}{2})\|\pi_{t+1}-\pi_t\|_{2}^{2}.
    \end{align}
\end{proof}

\begin{lem}
	\label{lem:ubg}
    Denote $G(\pi_t,\lambda) = \frac{1}{\eta_\pi}\left[\pi_t-Proj(\pi_t+\eta_\pi\nabla_\pi J_\nu(\pi_t,\lambda))\right]$ and let $\eta_\pi = \frac{1}{L}$, we have:
    \begin{align}
        \min_{t=0,1,...,T-1}\|G(\pi_t,\lambda)\|_{2}\le\sqrt{\frac{2L(J_\nu(\pi^*_\lambda,\lambda)-J_\nu(\pi_0,\lambda))}{T}}.
    \end{align}
\end{lem}

\begin{proof}
    By Lemma~\ref{lem:increase}, we have $J_\nu(\pi_{t+1},\lambda)-J_\nu(\pi_t,\lambda)\ge\frac{1}{2L}\|G(\pi_t,\lambda)\|_{2}^{2}$. By summing over $t$, we have:
    \begin{align}
        \min_{t=0,1,...,T-1}\|G(\pi_t,\lambda)\|^{2}_{2}\le\frac{2L(J_\nu(\pi^*_\lambda,\lambda)-J_\nu(\pi_0,\lambda))}{T}
    \end{align}
\end{proof}

\begin{lem}
    \label{lem:eps-stationary}
    If $J_\nu(\pi,\lambda)$ is L-smooth w.r.t $\pi$ for a fixed $\lambda$, the following inequality holds for projected gradient descent when $\|G(\pi_t,\lambda)\|_{2}\le\varepsilon$:
    \begin{align}
        \max_{\pi+\delta\in\Delta(\A)^{|\S|},\|\delta\|_{2}\le 1}\langle\delta, \nabla_\pi J_\nu(\pi_{t+1},\lambda)\rangle\le\varepsilon(\eta_\pi L+1)
    \end{align}  
\end{lem}
\noindent\textbf{Remark}: Lemma~\ref{lem:eps-stationary} can be referred from \cite{agarwal2019optimality}.

\begin{thm}
    \label{thm:converge}
    Let $\eta_\pi = 1/L_\lambda$ and Assumption~\ref{ass:ubnu} hold, we have:
    \begin{align}
        J_\nu(\pi^*_\lambda,\lambda)-J_\nu(\pi_T,\lambda)\le4\rho_\nu\sqrt{|\SM|}\left(\sqrt{\frac{2L_\lambda(J_\nu(\pi^*_\lambda,\lambda)-J_\nu(\pi_0,\lambda))}{T}}\right)
    \end{align}
\end{thm}

\begin{proof}
	By Lemma~\ref{lem:diff}, we have:
	\begin{align}
		J_\nu(\pi^*_\lambda,\lambda)-J_\nu(\pi,\lambda) &= \frac{1}{1-\gamma}\E_{d_{\pi^*_\lambda,\nu}}\langle\pi^*_\lambda(\cdot|s), A^{\pi}_{\lambda}(s,\cdot)-\lambda\Omega(\pi_\lambda^*(\cdot|s))\rangle\notag\\
        &\le\frac{1}{1-\gamma}\E_{d_{\pi^*_\lambda,\nu}}\max_{\tilde{\pi}(\cdot|s)}\langle\tilde{\pi}(\cdot|s), A^{\pi}_{\lambda}(s,\cdot)-\lambda\Omega(\tilde{\pi}(\cdot|s))\rangle\notag\\
        &\le\frac{1}{1-\gamma}\left[\max_{s}\frac{d_{\pi^*_\lambda,\nu}(s)}{d_{\pi,\nu}(s)}\right]\E_{d_{\pi,\nu}}\max_{\tilde{\pi}(\cdot|s)}\langle\tilde{\pi}(\cdot|s), A^{\pi}_{\lambda}(s,\cdot)-\lambda\Omega(\tilde{\pi}(\cdot|s))\rangle\notag\\
        &\le\frac{\rho_\nu}{1-\gamma}\max_{\tilde{\pi}}\left[\E_{d_{\pi,\nu},\tilde{\pi}}A^{\pi}_\lambda(s,a)-\lambda\Omega(\tilde{\pi}(\cdot|s))\right]
	\end{align}
where the second inequality holds by $\max_{\tilde{\pi}(\cdot|s)}\langle\tilde{\pi}(\cdot|s), A^{\pi}_{\lambda}(s,\cdot)-\lambda\Omega(\tilde{\pi}(\cdot|s))\rangle\ge0$ (let $\widetilde{\pi}=\pi$), and the final step follows $\tilde{\pi}(\cdot|s)$ are independent with each state. Next we turn to upper bound $\E_{d_{\pi,\nu},\tilde{\pi}}A^{\pi}_\lambda(s,a)-\lambda\Omega(\tilde{\pi}(\cdot|s))$.
    \begin{align}
		\E_{d_{\pi,\nu},\tilde{\pi}}\left[A^{\pi}_\lambda(s,a)-\lambda\Omega(\tilde{\pi}(\cdot|s))\right]&\overset{(a)}{=}\E_{d_{\pi,\nu},\tilde{\pi}}\left[A^{\pi}_\lambda(s,a)-\lambda\Omega(\tilde{\pi}(\cdot|s))\right]-\E_{d_{\pi,\nu},\pi}\left[A^{\pi}_\lambda(s,a)-\lambda\Omega(\pi(\cdot|s))\right]\notag\\
		&\overset{(b)}{=}\E_{d_{\pi,\nu},\tilde{\pi}}\left[Q^{\pi}_\lambda(s,a)-\lambda\Omega(\tilde{\pi}(\cdot|s))\right]-\E_{d_{\pi,\nu},\pi}\left[Q^{\pi}_\lambda(s,a)-\lambda\Omega(\pi(\cdot|s))\right]\notag\\
		&=\E_{d_{\pi,\nu}}\left[\langle\tilde{\pi}(\cdot|s)-\pi(\cdot|s), Q_{\lambda}^{\pi}(s, \cdot) \rangle-\lambda(\Omega(\tilde{\pi}(\cdot|s))-\Omega(\pi(\cdot|s)))\right]\notag\\
		&\overset{(c)}{\le}\E_{d_{\pi,\nu}}\left[\langle\tilde{\pi}(\cdot|s)-\pi(\cdot|s), Q_{\lambda}^{\pi}(s,\cdot) \rangle-\lambda\langle\tilde{\pi}(\cdot|s)-\pi(\cdot|s), \nabla\Omega(\pi(\cdot|s))\rangle\right]\notag\\
		&\overset{(d)}{=}(1-\gamma)\langle\tilde{\pi}-\pi, \nabla J_\nu(\pi,\lambda)\rangle
    \end{align}
where (a) follows from $\sum_{a}\pi(a|s)A^{\pi}_\lambda(s,a)-\lambda\Omega(\pi(\cdot|s))=0$, (b) follows from $A^{\pi}_\lambda(s,a)=Q^{\pi}_{\lambda}(s,a)-V^{\pi}_\lambda(s)$, (c) follows since $\Omega(\pi)$ is a convex function, and (d) follows from the definition of $\nabla J_\nu(\pi,\lambda)$.Then we obtain an upper bound of $J_\nu(\pi^*_\lambda,\lambda)-J_\nu(\pi,\lambda)$:
	\begin{align}
		J_\nu(\pi^*_\lambda,\lambda)-J_\nu(\pi,\lambda)&\le\rho_\nu \max_{\tilde{\pi}}\langle\tilde{\pi}-\pi, \nabla J_\nu(\pi,\lambda)\rangle\notag\\
        &\le2\rho_\nu\sqrt{|\SM|}\max_{\pi+\delta\in\Delta(\AM)^{|\SM|},\|\delta\|_{2}\le 1}\langle\delta, \nabla_\pi J_\nu(\pi,\lambda)\rangle
	\end{align}
where the final inequality holds as $\|\tilde{\pi}-\pi\|_{2}\le2\sqrt{|\SM|}$ and:
	\begin{align}
		\max_{\tilde{\pi}}\langle\tilde{\pi}-\pi, \nabla J(\pi,\lambda)\rangle\le2\sqrt{|\SM|}\max_{\pi+\delta\in\Delta(\AM)^{|\SM|},\|\delta\|_{2}\le 1}\langle\delta, \nabla_\pi J_\nu(\pi,\lambda)\rangle\notag
	\end{align}

    By Lemma~\ref{lem:eps-stationary} and Lemma~\ref{lem:ubg}, we have:
    \begin{align}
        \min_{t=0,1,...,T-1}\max_{\pi_t+\delta\in\Delta(\AM)^{|\SM|},\|\delta\|_{2}\le 1}\langle\delta, \nabla_\pi J_\nu(\pi_{t+1},\lambda)\rangle\le2\sqrt{\frac{2L_\lambda(J_\nu(\pi^*_\lambda,\lambda)-J_\nu(\pi_0,\lambda))}{T}}
    \end{align}

    Gathering all these results together, we have:
    \begin{align}
		J_\nu(\pi^*_\lambda,\lambda)-J_\nu(\pi_T,\lambda)&\le\min_{t=0,1,...,T-1} J_\nu(\pi^*_\lambda,\lambda)-J_\nu(\pi_{t+1},\lambda)\notag\\
		&\le4\rho_\nu\sqrt{|\SM|}\left(\sqrt{\frac{2L_\lambda(J_\nu(\pi^*_\lambda,\lambda)-J_\nu(\pi_0,\lambda))}{T}}\right)
    \end{align}
    where the first inequality holds by Lemma~\ref{lem:increase}.
\end{proof}

\subsubsection{Proof of Corollary~\ref{cor:pgaahcc}}
\begin{proof}
    By Equation~(\ref{eq:dt}) and Theorem~\ref{thm:pga}, we have:
    \begin{align}
        J_\nu(\pi^*_{\lambda_t},\lambda_t)-J_\nu(\hpi_{t+1}, \lambda_t)&\le4\rho_\nu\sqrt{|\SM|}\sqrt{\frac{2L_{\lambda_t} D_t}{T}}\notag\\
        &\le4\rho_\nu\sqrt{|\SM|}\sqrt{\frac{8\lambda_0 L_{\lambda_t} \Cphi}{2^t(1-\gamma)T}}
    \end{align}
    Let the RHS of above inequality equals $\frac{\lambda_0}{2^t}\frac{C_\phi}{1-\gamma}$, then the total time satisfies $\propii$ at timestep $t$ is at most:
    \begin{align}
        \tm(\lambda_t)\le\frac{128|\SM|\rho_\nu^2 L_{\lambda_t}(1-\gamma)}{\lambda_0 \Cphi}2^t
    \end{align}
\end{proof}

\subsubsection{Proof of Theorem~\ref{thm:pgared}}
\label{prf:thmpgared}
\begin{proof}
    In order to obtain an $\varepsilon$-optimal policy w.r.t. initial distribution $\mu$, we have to get an $\varepsilon/\rho$-optimal policy w.r.t. initial distribution $\nu$ at first by Theorem~\ref{thm:diffval}. By Corollary~\ref{cor:pgaahcc}, we can obtain an $\varepsilon$-optimal policy in total time:
	\begin{align}
		\sum_{t=0}^{T-1}\tm(\lambda_t)\le\sum_{t=0}^{T-1}\frac{128|\SM|\rho_\nu^2 L_{\lambda_t}(1-\gamma)}{\lambda_0 \Cphi}2^t
	\end{align}
	Note that $L_\lambda\le\frac{4|\AM|}{(1-\gamma)^3}+\lambda\cdot\frac{4|\AM|\widetilde{C}_\Phi}{(1-\gamma)^3}$, where $\widetilde{C}_\Phi$ is dependent on $\Cphi^{(0,1,2)}$, and $T=O(\log_2\frac{6\rho\lambda_0 \Cphi}{\varepsilon(1-\gamma)})$, thus the total time is:
	\begin{align}
	    \sum_{t=0}^{T-1}\tm(\lambda_t)&\le\frac{128|\SM|\rho_\nu^2(1-\gamma)}{\lambda_0 \Cphi}\sum_{t=0}^{T-1}L_{\lambda_t}2^t\notag\\
	    &\le\frac{128|\SM||\AM|\rho_\nu^2}{(1-\gamma)^2\lambda_0 \Cphi}(2^T+\lambda_0 C_2 T)\notag\\
	    &=O\left(\frac{|\SM||\AM|\rho\rho_\nu^2\widetilde{C}_\Phi}{\varepsilon(1-\gamma)^3}\right)
	\end{align}
\end{proof}

\subsection{Sampling}
\label{apx:sample}
To simplify, we denote $\widetilde{\nabla}J_\nu(\pi,\lambda)$ as an estimator of $\nabla J_\nu(\pi,\lambda)$, where $\nabla_{s,a}J_\nu(\pi,\lambda)=\frac{1}{1-\gamma}d_{\pi,\nu}(s)(Q_\lambda^{\pi}(s,a)-\lambda\nabla_{s,a}\Omega(\pi(\cdot|s))$.
\begin{lem}
    \label{lem:increase_sample}
    Denote $\pi_{t+1}=Proj(\pi_t+\eta_\pi\widetilde{\nabla}_\pi J_\nu(\pi_t,\lambda))$, we have the following improvement guarantee:
    \begin{align}
        J_\nu(\pi_{t+1},\lambda)-J_\nu(\pi_t,\lambda)\ge\frac{2-\eta_\pi L_\lambda}{2\eta_\pi}\|\pi_{t+1}-\pi_t\|_{2}^2+\varepsilon_t,
    \end{align}
    where $L_\lambda$ is the smoothness coefficient of $J_\nu(\pi,\lambda)$ and $\varepsilon_t = \langle\nabla J_\nu(\pi_t,\lambda)-\widetilde{\nabla}J_\nu(\pi_t,\lambda),\pi_{t+1}-\pi_t\rangle$.
\end{lem}
\begin{proof}
    By L-smooth, we have:
    \begin{align}
        J_\nu(\pi_{t+1},\lambda)\ge J_\nu(\pi_t,\lambda)+\langle\widetilde{\nabla}J_\nu(\pi_t,\lambda),\pi_{t+1}-\pi_t\rangle-\frac{L}{2}\|\pi_{t+1}-\pi_t\|_{2}^{2}+\varepsilon_t
    \end{align}
    By first-order condition, we have:
    \begin{align}
        \langle\pi_{t+1}-\pi_t-\eta\widetilde{\nabla}J_\nu(\pi_t,\lambda),\pi_{t+1}-\pi_t\rangle\le0
    \end{align}
    Combining these two together, we obtain desired result.
\end{proof}

It's not easy to obtain exact knowledge of $d_{\pi,\nu}$. But, we can obtain an another  $d_{\pi,\nu,K}$ to approximate $d_{\pi,\nu}$. According to \citet{kakade2002approximately, shani2019adaptive}, we draw a start state $s$ from $\nu(s)$, and accept it as initial state for trajectory simulation with probabilty $1-\gamma$. Otherwise, we transits it to next state with probability $\gamma$. We repeat the process until an acceptance is made within time $K$. In mathematical form, $d_{\pi,\nu,K}(s)=(1-\gamma)\nu(s)+(1-\gamma)\sum_{t=1}^{K-1}\gamma^t P(s_t=s|\nu, \pi)+\gamma^K P(s_T=s|\nu,\pi)$.

From a state-action pair $(s_0,a_0)\sim d_{\pi,\nu, K}\times\mathcal{U}(\AM)$, we can simulate a truncated trajectory $(s_0, a_0, s_1, a_1, ..., s_{K-1}, a_{K-1})$ with given policy $\pi$. Denote $\widehat{Q}_\lambda^{\pi}(s_0, a_0)=R(s_0, a_0)+\sum_{t=1}^{K-1}\gamma^t(R(s_t, a_t)-\lambda\Omega(\pi(\cdot|s_t)))$, which is an unbiased estimator of 
\begin{align}
    \overline{Q}_\lambda^{\pi}(s_0, a_0)=\mathbb{E}_{\pi,\mathbb{P}}\left[R(s_0,a_0)+\sum_{t=1}^{K-1}\gamma^t(R(s_t, a_t)-\lambda\Omega(\pi(\cdot|s_t)))\right]
\end{align}
Note that $\overline{Q}_\lambda^\pi$ could be close enough with $Q_\lambda^{\pi}$ when $K$ is sufficiently large. Thus, we can derive an estimator of $\nabla_{s,a} J_\nu(\pi,\lambda)$ , with N independent trajectories, as:
\begin{align}
    \widetilde{\nabla}_{s,a}J_\nu(\pi,\lambda)=\frac{|\AM|}{1-\gamma}\frac{1}{N}\sum_{n=1}^{N}\left[\widehat{Q}_\lambda^{\pi}(s_{0,n}, a_{0,n})-\lambda\nabla_{s,a}\Omega(\pi(\cdot|s_{0,n}))\right]I\{s_{0,n}=s, a_{0,n}=a\}:\overset{\Delta}{=}\frac{1}{N}\sum_{n=1}^{N}\widehat{X}_n(s,a)
\end{align}

\begin{lem}
\label{lem:trunc_hoeffding}
For fixed $\widetilde{\pi}$, we have:
\begin{align}
        P\left(\left|\langle\widetilde{\nabla}J_\nu(\pi,\lambda)-\mathbb{E}\widetilde{\nabla}J_\nu(\pi,\lambda), \widetilde{\pi}-\pi\rangle\right|\ge\varepsilon\right)\le2\exp\left(-\frac{N\varepsilon^2(1-\gamma)^4}{2|\AM|^2(1+\lambda\widetilde{C}_\Phi)^2}\right)    
\end{align}
where $\mathbb{E}\widetilde{\nabla}_{s,a}J_\nu(\pi,\lambda)=\frac{1}{1-\gamma}d_{\pi,\nu,K}(s)\left(\overline{Q}_\lambda^{\pi}(s,a)-\lambda\nabla_{s,a}\Omega(\pi(\cdot|s))\right)$.
\end{lem}
\begin{proof}
    First of all, we proof that $\langle\widehat{X}_n,\widetilde{\pi}-\pi \rangle$ is bounded:
    \begin{align}
        |\langle\widehat{X}_n,\widetilde{\pi}-\pi \rangle|&=\left|\sum_{s\in\SM}\sum_{a\in\AM}\widehat{X}_{n}(s,a)(\widetilde{\pi}(a|s)-\pi(a|s))\right|\notag\\
        &\le\frac{|\AM|}{1-\gamma}\sum_{s\in\SM}\sum_{a\in\AM}\|\widehat{Q}_\lambda^\pi-\lambda\nabla\Omega(\pi)\|_{\infty}|\widetilde{\pi}(a|s)-\pi(a|s)|I\{s_{0,k}=s, a_{0,k}=a\}\notag\\
        &\le\frac{|\AM|}{1-\gamma}\|\widehat{Q}_\lambda^\pi-\lambda\nabla\Omega(\pi)\|_{\infty}\sum_{s\in\SM}\sum_{a\in\AM}|\widetilde{\pi}(a|s)-\pi(a|s)|I\{s_{0,k}=s\}\notag\\
        &\le\frac{2|\AM|(1+\lambda\widetilde{\Cphi})}{(1-\gamma)^2}
    \end{align}
    By Hoeffding's inequality, we have:
    \begin{align}
        P\left(\left|\langle\widetilde{\nabla}J_\nu(\pi,\lambda)-\mathbb{E}\widetilde{\nabla}J_\nu(\pi,\lambda), \widetilde{\pi}-\pi\rangle\right|\ge\varepsilon\right)\le2\exp\left(-\frac{N\varepsilon^2(1-\gamma)^4}{2|\AM|^2(1+\lambda\widetilde{C}_\Phi)^2}\right)
    \end{align}
\end{proof}

\begin{lem}
\label{lem:trunc_bias}
For any $\pi,\widetilde{\pi}$, we have:
\begin{align}
    \left|\langle\nabla J_\nu(\pi,\lambda)-\mathbb{E}\widetilde{\nabla}J_\nu(\pi,\lambda), \widetilde{\pi}-\pi\rangle\right|\le\frac{6\gamma^K(1+\lambda\Cphi)}{(1-\gamma)^2}
\end{align}
\end{lem}
\begin{proof}
    By definition, we decompose the error $\left|\langle\nabla J_\nu(\pi,\lambda)-\mathbb{E}\widetilde{\nabla}J_\nu(\pi,\lambda), \widetilde{\pi}-\pi\rangle\right|$ into two parts $I_1,I_2$:
    \begin{align}
        &I_1=\left|\sum_{s\in\SM}\frac{1}{1-\gamma}d_{\pi,\nu}(s)\langle Q_\lambda^\pi(s,\cdot)-\overline{Q}_\lambda^{\pi}(s,\cdot), \widetilde{\pi}(\cdot|s)-\pi(\cdot|s)\rangle\right|\notag\\
        &I_2=\left|\sum_{s\in\SM}\frac{1}{1-\gamma}\left(d_{\pi,\nu}(s)-d_{\pi,\nu,T}(s)\right)\langle\overline{Q}_\lambda^\pi(s,\cdot), \widetilde{\pi}(\cdot|s)-\pi(\cdot|s)\rangle\right|\notag
    \end{align}
    For $I_1$, we have:
    \begin{align}
        I_1&\le\sum_{s\in\SM}\frac{1}{1-\gamma}d_{\pi,\nu}(s)\left|\langle\overline{Q}_\lambda^\pi(s,\cdot)-Q_\lambda^{\pi}(s,\cdot), \widetilde{\pi}(\cdot|s)-\pi(\cdot|s)\rangle\right|\notag\\
        &\le\frac{1}{1-\gamma}\sum_{s\in\SM}d_{\pi,\nu}(s)\left\|\overline{Q}_{\lambda}^{\pi}(s,\cdot)-Q_{\lambda}^{\pi}(s,\cdot)\right\|_{\infty}\left\|\widetilde{\pi}(\cdot|s)-\pi(\cdot|s)\right\|_{1}\notag\\
        &\le\frac{2}{1-\gamma}\left\|\overline{Q}_\lambda^\pi-Q_\lambda^\pi\right\|_{\infty}\notag\\
        &\le\frac{2\gamma^K(1+\lambda\Cphi)}{(1-\gamma)^2}
    \end{align}
    For $I_2$, we mainly concern the difference between $d_{\pi,\nu}$ and $d_{\pi,\nu,T}$:
    \begin{align}
        \sum_{s\in\SM}\left|d_{\pi,\nu}(s)-d_{\pi,\nu,K}(s)\right|&=\sum_{s\in\SM}\left|(1-\gamma)\sum_{t=K}^{\infty}\gamma^t P(s_t=s|\nu,\pi)-\gamma^K P(s_K=s|\nu,\pi)\right|\notag\\
        &\le(1-\gamma)\sum_{s\in\SM}\sum_{t=K}^{\infty}\gamma^t P(s_t=s|\nu,\pi)+\gamma^K\sum_{s\in\SM}P(s_K=s|\nu,\pi)\notag\\
        &=2\gamma^K
    \end{align}
    Thus, we have:
    \begin{align}
        I_2&\le\sum_{s\in\SM}\frac{1}{1-\gamma}\left|d_{\pi,\nu}(s)-d_{\pi,\nu,T}(s)\right|\left|\langle\overline{Q}_{\lambda}^{\pi}(s,\cdot),\widetilde{\pi}(\cdot|s)-\pi(\cdot|s)\rangle\right|\notag\\
        &\le\frac{2}{1-\gamma}\sum_{s\in\SM}\left|d_{\pi,\nu}(s)-d_{\pi,\nu,T}(s)\right|\|\overline{Q}_{\lambda}^{\pi}\|_{\infty}\notag\\
        &\le\frac{4\gamma^K(1+\lambda\Cphi)}{(1-\gamma)^2}
    \end{align}
    Combining $I_1$, $I_2$ together, we have:
    \begin{align}
        \left|\langle\nabla J_\nu(\pi,\lambda)-\mathbb{E}\widetilde{\nabla}J_\nu(\pi,\lambda), \widetilde{\pi}-\pi\rangle\right|\le\frac{6\gamma^K(1+\lambda\Cphi)}{(1-\gamma)^2}
    \end{align}
\end{proof}

\begin{lem}
\label{lem:err_t}
Fixing randomness before $\pi_t$ and letting $K=\log\frac{12(1+\lambda\Cphi)}{\varepsilon(1-\gamma)^2}/\log\frac{1}{\gamma}$, with probability $1-\delta$ and trajectory sample size $N=\frac{8|\AM|^2(1+\lambda\widetilde{C}_\Phi)^2}{\varepsilon^2(1-\gamma)^4}\left(|\SM|\log|\AM|+\log\frac{2}{\delta}\right)$, we have:
\begin{align}
|\varepsilon_t|\le\varepsilon
\end{align}
\end{lem}
\begin{proof}
    By definition of $\varepsilon_t$, we have:
    \begin{align}
        P\left(|\varepsilon_t|>\varepsilon\right)&=P\left(\left|\langle\nabla J_\nu(\pi_t,\lambda)-\widetilde{\nabla}J_\nu(\pi_t,\lambda), \pi_{t+1}-\pi_t\rangle\right|>\varepsilon\right)\notag\\
        &\le P\left(\max_{\widetilde{\pi}}\left|\langle\nabla J_\nu(\pi_t,\lambda)-\widetilde{\nabla}J_\nu(\pi_t,\lambda), \widetilde{\pi}-\pi_t\rangle\right|>\varepsilon\right)\notag\\
        &\le |\AM|^{|\SM|}P\left(\left|\langle\nabla J_\nu(\pi_t,\lambda)-\widetilde{\nabla}J_\nu(\pi_t,\lambda), \widetilde{\pi}-\pi_t\right\rangle|>\varepsilon\right)
    \end{align}
    where the last inequality holds that $\widetilde{\pi}$ is the maximum point of a linear function, which is deterministic, and the number of deterministic policies are $|\AM|^{|\SM|}$. Thus, we continue to bound the following term:
    \begin{align}
        &P\left(\left|\langle\nabla J_\nu(\pi_t,\lambda)-\widetilde{\nabla}J_\nu(\pi_t,\lambda), \widetilde{\pi}-\pi_t\right\rangle|>\varepsilon\right)\notag\\
        =&P\left(\left|\langle\nabla J_\nu(\pi_t,\lambda)-\mathbb{E}\widetilde{\nabla}J_\nu(\pi_t,\lambda)+\mathbb{E}\widetilde{\nabla}J_\nu(\pi_t,\lambda)-\widetilde{\nabla}J_\nu(\pi_t,\lambda),\widetilde{\pi}-\pi_t\rangle\right|>\varepsilon\right)\notag\\
        \le &P\left(\left|\langle\nabla J_\nu(\pi_t,\lambda)-\mathbb{E}\widetilde{\nabla}J_\nu(\pi_t,\lambda), \widetilde{\pi}-\pi_t\rangle\right|+\left|\langle\mathbb{E}\widetilde{\nabla}J_\nu(\pi_t,\lambda)-\widetilde{\nabla}J_\nu(\pi_t,\lambda), \widetilde{\pi}-\pi_t\rangle\right|>\varepsilon\right)
    \end{align}
    By setting $K=\log\frac{12(1+\lambda\Cphi)}{\varepsilon(1-\gamma)^2}/\log\frac{1}{\gamma}$ in Lemma~\ref{lem:trunc_bias}, we have:
    \begin{align}
        P\left(\left|\langle\nabla J_\nu(\pi_t,\lambda)-\widetilde{\nabla}J_\nu(\pi_t,\lambda), \widetilde{\pi}-\pi_t\right\rangle|>\varepsilon\right)&\le P\left(\left|\langle\mathbb{E}\widetilde{\nabla} J_\nu(\pi_t,\lambda)-\widetilde{\nabla}J_\nu(\pi_t,\lambda), \widetilde{\pi}-\pi_t\right\rangle|>\frac{\varepsilon}{2}\right)\notag\\
        &\le2\exp\left(-\frac{N\varepsilon^2(1-\gamma)^4}{8|\AM|^2(1+\lambda\widetilde{C}_\Phi)^2}\right)    
    \end{align}
    where the last inequality follows from Lemma~\ref{lem:trunc_hoeffding}. Combining all these together and taking sample size $N=\frac{8|\AM|^2(1+\lambda\widetilde{C}_\Phi)^2}{\varepsilon^2(1-\gamma)^4}\left(|\SM|\log|\AM|+\log\frac{2}{\delta}\right)$, we obtain:
    \begin{align}
        P\left(|\varepsilon_t|>\varepsilon\right)\le\delta
    \end{align}
\end{proof}

\begin{lem}
Taking $\eta=1/L_\lambda$, with probability $1-\delta$ and the number of trajectories from $t=0,...,T-1$ being $\frac{8|\AM|^2(1+\lambda\widetilde{C}_\Phi)^2 T}{\varepsilon^2(1-\gamma)^4}\left(|\SM|\log|\AM|+\log\frac{2T}{\delta}\right)$, we have:
\begin{align}
    \min_{t=0,...,T-1}\|G(\pi_t,\lambda )\|_{2}^{2}\le\frac{2L_{\lambda}(J_\nu(\pi_\lambda^*,\lambda)-J_\nu(\pi_0, \lambda))}{T}+\varepsilon
\end{align}
where $G(\pi_t,\lambda)=(\pi_{t+1}-\pi_t)/\eta$.
\end{lem}

\begin{proof}
    By Lemma~\ref{lem:increase_sample}, we have:
    \begin{align}
        \min_{t=0,...,T-1}\|G(\pi_t,\lambda )\|_{2}^{2}\le\frac{2L_{\lambda}(J_\nu(\pi_\lambda^*,\lambda)-J_\nu(\pi_0, \lambda))}{T}-\frac{\sum_{t=0}^{T-1}\varepsilon_t}{T}
    \end{align}
    Besides, denote $\mathcal{F}_{t+1}=\sigma\left(\mathcal{F}_t\cup\sigma(\{X_{t,n}\}_{n=1}^{N})\right)$ and $\mathcal{F}_0$ contains all information before $\pi_0$, where $X_{n,t}$ are i.i.d random variables drawn from $\pi_t$ and environment. Thus, we have:
    \begin{align}
        P\left(\left|\frac{\sum_{t=0}^{T-1}\varepsilon_t}{T}\right|>\varepsilon\right)&\le\sum_{t=0}^{T-1}P\left(|\varepsilon_t|>\varepsilon\right)\notag\\
        &=\sum_{t=0}^{T-1}\mathbb{E}\left[P(|\varepsilon_t|>\varepsilon|\mathcal{F}_t)\right]\notag\\
        &\le\delta
    \end{align}
    where the last inequality follows from Lemma~\ref{lem:err_t} with number of trajectories being  $N_t=\frac{8|\AM|^2(1+\lambda\widetilde{C}_\Phi)^2}{\varepsilon^2(1-\gamma)^4}\left(|\SM|\log|\AM|+\log\frac{2T}{\delta}\right)$ at each step $t=0,...,T-1$.
\end{proof}

\begin{thm}
\label{thm:converge_sample}
Fixing $\lambda$ and taking $\eta=1/L_\lambda$, with probability $1-\delta$ and number of trajectories $\widetilde{O}\left(\frac{|\SM||\AM|^2(1+\lambda\widetilde{C}_\Phi)^2}{\varepsilon^2(1-\gamma)^4}\right)$ at each time-step $t$, we have:
\begin{align}
\min_{t=0,...,T-1}J_\nu(\pi_{\lambda}^*,\lambda)-J_\nu(\pi_{t+1},\lambda)\le4\rho_\nu\sqrt{|\SM|}\sqrt{\frac{L_{\lambda}(J(\pi_\lambda^*,\lambda)-J(\pi_0, \lambda))}{T}+\varepsilon}
\end{align}
In order to make the LHS equaling with $\varepsilon'$, the total iteration complexity could be $T=O\left(\frac{|\SM||\AM|\rho_\nu^2(1+\lambda\widetilde{C}_\Phi)^2}{\varepsilon'^2(1-\gamma)^4}\right)$ and the total number of sampled trajectories is $\widetilde{O}\left(\frac{|\SM|^4|\AM|^3\rho_\nu^6(1+\lambda\widetilde{C}_\Phi)^4}{\varepsilon'^6(1-\gamma)^{8}}\right)$
\end{thm}

\begin{proof}
    By proof of Theorem~\ref{thm:converge}, with probability $1-\delta$, we have:
    \begin{align}
        \min_{t=0,...,T-1}J_\nu(\pi_\lambda^*,\lambda)-J_\nu(\pi_{t+1},\lambda)&\le4\rho_\nu\sqrt{|\SM|}\min_{t=0,...,T-1}\|G(\pi_t,\lambda)\|_{2}\notag\\
        &\le4\rho_\nu\sqrt{|\SM|}\sqrt{\frac{2L_{\lambda}(J(\pi_\lambda^*,\lambda)-J(\pi_0, \lambda))}{T}+\varepsilon}\notag\\
        &\le4\rho_\nu\sqrt{|\SM|}\left(\sqrt{\frac{2L_{\lambda}(J(\pi_\lambda^*,\lambda)-J(\pi_0, \lambda))}{T}}+\sqrt{\varepsilon}\right)
    \end{align}
    By taking $4\rho_\nu\sqrt{|\SM|}\sqrt{\varepsilon}\le\frac{\varepsilon'}{2}$, the number of trajectories at each time-step is $\widetilde{O}\left(\frac{|\SM|^3|\AM|^2\rho_\nu^4(1+\lambda\widetilde{C}_\Phi)^2}{\varepsilon'^4(1-\gamma)^4}\right)$. 
    Besides, in order to make the LHS less than $\varepsilon'$, then $T=O\left(\frac{|\SM||\AM|\rho_\nu^2(1+\lambda\widetilde{C}_\Phi)^2}{\varepsilon'^2(1-\gamma)^4}\right)$ is enough. The total number of trajectories could be $\widetilde{O}\left(\frac{|\SM|^4|\AM|^3\rho_\nu^6(1+\lambda\widetilde{C}_\Phi)^4}{\varepsilon'^6(1-\gamma)^{8}}\right)$.
\end{proof}

Finally, we consider combining Algorithm~\ref{alg:ar} with Projected Gradient Ascent serving as sub-solver and only samples can be accessed. Besides, we also assume we can find $\arg\min_{t=0,...,T-1}J(\pi_\lambda^*,\lambda)-J(\pi_t,\lambda)$ firstly, and later we will relax this assumption.

\begin{lem}
\label{lem:condition_sample}
At each time-step $t$ with $\lambda_t$, taking $\eta=1/L_{\lambda_t}$, $T_t=\frac{128|\SM|\rho_\nu^2 L_{\lambda_t}(1-\gamma)}{\lambda_0\Cphi}2^t$ and $\varepsilon=\frac{\widehat{\varepsilon}^2}{16|\SM|\rho_\nu^2}$, with probability $1-\delta_t$ we have:
\begin{align}
    \min_{k=0,...,K-1}J_\nu(\pi_{\lambda_t}^*,\lambda)-J_\nu(\pi_{k+1},\lambda)\le\frac{\lambda_0\Cphi}{2^t(1-\gamma)}+\widehat{\varepsilon}
\end{align}
\end{lem}

\begin{proof}
    By Theorem~\ref{thm:converge_sample}, we have:
    \begin{align}
        \min_{k=0,...,T_t-1}J_\nu(\pi_{\lambda_t}^*,\lambda_t)-J_\nu(\pi_{k+1},\lambda_t)\le4\rho_\nu\sqrt{|\SM|}\sqrt{\frac{2L_{\lambda_t}D_t}{T_t}+\varepsilon}\le4\rho_\nu\sqrt{|\SM|}\sqrt{\frac{8\lambda_0 L_{\lambda_t}\Cphi}{2^t(1-\gamma)T_t}+\frac{2L_{\lambda_t}\widehat{\varepsilon}}{T_t}+\varepsilon}\notag
    \end{align}
    By taking $T_t=\frac{128|\SM|\rho_\nu^2 L_{\lambda_t}(1-\gamma)}{\lambda_0\Cphi}2^t$ and $\varepsilon=\frac{\widehat{\varepsilon}^2}{16|\SM|\rho_\nu^2}$, we have:
    \begin{align}
        \min_{k=0,...,T_t-1}J_\nu(\pi_{\lambda_t}^*,\lambda)-J_\nu(\pi_{k+1},\lambda)\le\frac{\lambda_0\Cphi}{2^t(1-\gamma)}+\widehat{\varepsilon}
    \end{align}
\end{proof}

\begin{thm}
Suppose we can obtain $\widehat{\pi}_{t+1}$ attaining minimum of $J_\nu(\pi_{\lambda_t}^*,\lambda_t)-J_\nu(\pi_{k+1}, \lambda_t)$ over $k=0,...,T_t-1$ at each time-step $t$. To obtain an $\varepsilon$-optimnal policy w.r.t initial distribution $\mu$, with probability $1-\delta$, the total iteration complexity is $O\left(\frac{|\SM||\AM|\rho\rho_\nu^2\widetilde{C}_\Phi}{\varepsilon(1-\gamma)^3}\right)$ and the number of trajectories is $\widetilde{O}\left(\frac{|\SM|^4|\AM|^3\rho^5\rho_\nu^6(1+\lambda_0\widetilde{C}_\Phi)^6}{\varepsilon^5(1-\gamma)^{7}}\right)$.
\end{thm}

\begin{proof}
    By Lemma~\ref{lem:condition_sample}, the iteration complexity at each timestep $t$ is $\tm(\lambda_t)=\frac{128|\SM|\rho_\nu^2 L_{\lambda_t}(1-\gamma)}{\lambda_0 \Cphi}2^t$, which is the same as Corollary~\ref{cor:pgaahcc}. Thus, the total iteration complexity could be $O\left(\frac{|\SM||\AM|\rho\rho_\nu^2\widetilde{C}_\Phi}{\varepsilon(1-\gamma)^3}\right)$ by taking $T=O(\log_2\frac{12\rho\lambda_0 \Cphi}{\varepsilon(1-\gamma)})$. By Theorem~\ref{thm:reduction2} and Theorem~\ref{thm:diffval}, the final output of Algorithm~\ref{alg:ar} satisfies:
    \begin{align}
        V^*(\mu)-V^{\widehat{\pi}_T}(\mu)\le\frac{\varepsilon}{2}+\rho\widehat{\varepsilon}
    \end{align}
    In order to obatian an $\varepsilon$-optimal policy, the number of trajectories we have to sample at each time-step in Projected Gradient Ascent is $\widetilde{O}\left(\frac{|\SM|^3|\AM|^2\rho^4\rho_\nu^4(1+\lambda_0\widetilde{C}_\Phi)^2}{\varepsilon^4(1-\gamma)^4}\right)$ by Theorem~\ref{thm:converge_sample}. Thus, the total number of trajectories is $\widetilde{O}\left(\frac{|\SM|^4|\AM|^3\rho^5\rho_\nu^6(1+\lambda_0\widetilde{C}_\Phi)^3}{\varepsilon^5(1-\gamma)^{7}}\right)$, where we ignore logarithm terms, especially $\log\left(\frac{T\sum_{t=1}^{T}\tm(\lambda_t)}{\delta}\right)$.
\end{proof}

However, $\widehat{\pi}_t$ cannot be obtained directly by $J_\nu(\pi,\lambda)$. Thus we evaluate $J_\nu(\pi,\lambda)$ by re-sampling as $\widehat{J}_\nu(\pi,\lambda)$.

\begin{lem}
\label{lem:min}
Suppose $\{X_{i,t}\}_{i=1}^{N}$ are i.i.d and bounded by $[0,M]$ for each t. With sample complexity $N=\frac{M^2}{2\varepsilon^2}\log\frac{T}{\delta}$ and probability $1-\delta$,
\begin{align}
    \max_{t=0,...,T-1} \frac{1}{N}\sum_{i=1}^{N}X_{i,t}\ge\max_{t=0,...,T-1}\mu_t-\varepsilon
\end{align}
where $\mu_t$ is mean of $X_{i,t}$.
\end{lem}
\begin{proof}
    In fact we have:
    \begin{align}
        P\left(\max_{t=0,...,T-1} \frac{1}{N}\sum_{i=1}^{N}X_{i,t}<\max_{t=0,...,T-1}\mu_t-\varepsilon\right)&\le P\left(\exists \text{t, s.t. }\frac{1}{N}\sum_{i=1}^{N}X_{i,t}<\mu_t-\varepsilon \right)\notag\\
        &\le\sum_{t=0}^{T-1}P\left(\frac{1}{N}\sum_{i=1}^{N}X_{i,t}<\mu_t-\varepsilon\right)\notag\\
        &\le T\exp\left(-\frac{2\varepsilon^2 N}{M^2}\right)
    \end{align}
    Thus, taking $N=\frac{M^2}{2\varepsilon^2}\log\frac{T}{\delta}$ is enough.
\end{proof}

\begin{lem}
\label{lem:minall}
Under the same setting in Lemma~\ref{lem:min}, denote $\widehat{t}=\argmax_{t=0,...,T-1}\frac{1}{N}\sum_{i=1}^{N}X_{i,t}$, with probability $1-2\delta$, we have:
\begin{align}
    \mu_{\widehat{t}}\ge\max_{t=0,...,T-1}\mu_t-2\varepsilon
\end{align}
\end{lem}
\begin{proof}
    By Lemma~\ref{lem:min} and symmetric property, we have:
    \begin{align}
        P\left(\frac{1}{N}\sum_{i=1}^{N}X_{i,\widehat{t}}>\mu_{\widehat{t}}+\varepsilon\right)\le\frac{\delta}{T}
    \end{align}
    Thus, we have:
    \begin{align}
        P\left(\mu_{\widehat{t}}<\max_{t=0,...,T-1}\mu_t-2\varepsilon\right)&\le P\left(\frac{1}{N}\sum_{i=1}^{N}X_{i,\widehat{t}}>\mu_{\widehat{t}}+\varepsilon\right)+P\left(\max_{t=0,..,T-1}\mu_t-\varepsilon>\max_{t=0,..,T-1}\frac{1}{N}\sum_{i=1}^{N}X_{i,t}\right)\notag\\
        &\le\frac{1+T}{T}\delta\le2\delta
    \end{align}
\end{proof}

At time-step $t$, the output policies of Projected Gradient Ascent are $\pi_0,...,\pi_{T_t-1}$. For each policy $\pi_k$, we interact with environment to obtain $n$ trajectories $(s_{n,0},a_{n,0},...s_{n,K-1}, a_{n,K-1})$ from initial distribution $\nu$ and policy $\pi_k$, and evaluate $J_\nu(\pi_k, \lambda)$ by:
\begin{align}
    \widehat{J}(\pi_k,\lambda)=\frac{1}{N}\sum_{n=1}^{N}\widehat{Q}^{\pi_k}_{\lambda}(s_{n,0},a_{n,0})-\lambda\Omega(\pi_k(\cdot|s_{n,0}))
\end{align}
where $\widehat{Q}^{\pi_k}_{\lambda}(s_{n,0},a_{n,0})=R(s_{n,0},a_{n,0})+\sum_{k=1}^{K-1}\gamma^t(R(s_{n,t}, a_{n,t})-\lambda\Omega(\pi_k(\cdot|s_{n,t})))$. Thus, the expectation of $\widehat{J}_\nu(\pi_k,\lambda)$ is:
\begin{align}
        \mathbb{E}\widehat{J}_\nu(\pi_k,\lambda)=\mathbb{E}_{\pi,\mathbb{P}}\left[\sum_{t=0}^{K-1}\gamma^t(R(s_t, a_t)-\lambda\Omega(\pi(\cdot|s_t)))\right]
\end{align}

\begin{thm}
At time-step $t$, with probability $1-2\delta$ and $N=\frac{2(1+\lambda\Cphi)^2}{\varepsilon^2(1-\gamma)^2}\log\frac{T_t}{\delta}$ for each policy $\pi_i$,
\begin{align}
    J_\nu(\pi_{\lambda_t}^*,\lambda_t)-J_\nu(\pi_{\widehat{i}},\lambda_t)\le \min_{i=1,...,T_t} J_\nu(\pi_{\lambda_t}^*,\lambda_t)-J_\nu(\pi_i,\lambda_t)+2\varepsilon
\end{align}
where $\widehat{i}=\argmax_{i=1,...,T_t} \widehat{J}_\nu(\pi_i,\lambda_t)$.
\end{thm}
\begin{proof}
    Note that $|\widehat{J}|\le\frac{1+\lambda\Cphi}{1-\gamma}$, by Lemma~\ref{lem:min} and setting $K=\log\frac{2(1+\lambda\Cphi)}{\varepsilon(1-\gamma)}/\log\frac{1}{\gamma}$, we have:
    \begin{align}
        P\left(\max_{i=1,...,T_t}\widehat{J}_\nu(\pi_i,\lambda_t)<\max_{i=1,...,T_t}J_\nu(\pi_i,\lambda_t) -\varepsilon\right)&\le\sum_{i=1}^{T_t}P\left(\widehat{J}_\nu(\pi_i,\lambda_t)<J_\nu(\pi_i,\lambda_t) -\varepsilon\right)\notag\\
        &\le\sum_{i=1}^{T_t}P\left(\widehat{J}_\nu(\pi_i,\lambda_t)<\mathbb{E}\widehat{J}_\nu(\pi_i,\lambda_t)-\varepsilon/2\right)\notag\\
        &\le T_t\exp\left(-\frac{\varepsilon^2 N(1-\gamma)^2}{2(1+\lambda\Cphi)^2}\right)\notag
    \end{align}
    where the second inequality holds by $J_\nu(\pi,\lambda)-\mathbb{E}\widehat{J}_\nu(\pi,\lambda)\le\frac{\gamma^K(1+\lambda\Cphi)}{1-\gamma}$. Thus, by Lemma~\ref{lem:minall}, we have:
    \begin{align}
        P\left(J_\nu(\pi_{\widehat{i}},\lambda_t)<\max_{i=1,...,T_t}J_\nu(\pi_i,\lambda_t)-2\varepsilon\right)\le2\delta
    \end{align}
\end{proof}

\section{Proof of Section~\ref{sec:minmax}}
\subsection{Proof of Theorem~\ref{thm:maxmax}}
\begin{proof}
	The first part is rather obvious. 
	Since the reward $r$ and the regularization function $\Omega$ are both uniformly bounded as defined, then $|J(\pi, \lambda)| \le | V^\pi(\mu) | +   \lambda  |\Phi^\pi(\mu) | \le \frac{R + \lambda \Cphi}{1-\gamma} < \infty$ for any fixed $\lambda$.
	Therefore, we have that $ \max_{\pi} J(\pi, \lambda)$ is finite and thus well-defined for any fixed $\lambda$.
	
	For the second part (here we don't assume $\Omega$ is bounded), it is an important observation that $J(\pi, \lambda)  =  V^\pi(\mu) - \lambda \Phi^\pi(\mu) $ increases in $\lambda$ since $\Omega$ is non-positive.
	Therefore, 
	\begin{equation*}
	\min_{\lambda \ge 0} J(\pi, \lambda) =J(\lambda, 0) = J(\pi) 
	\end{equation*}
	whose absolute value is bounded.
	By taking maximum in $\pi$ on both side of the last equality, we finish the first equality in~\eqref{eq:maxmax}.
	
	Let's focus on the second equality then.
	For simplicity, denote by $LHS =  \max_{\pi} \min_{\lambda \ge 0} J(\pi, \lambda)$ and $RHS= \min_{\lambda \ge 0}  \max_{\pi} J(\pi, \lambda)$.
	From previous discussion, we already have $LHS = J(\pi^*)$.
	For one thing, note that we always have 
	\[ 	\max_\pi \min_{\lambda\ge0} J(\pi,\lambda)\le\min_{\lambda\ge0}\max_\pi J(\pi,\lambda) , \]
	then $LHS \le RHS$.
	For the other thing,
	\begin{equation*}
        \max_{\pi}J(\pi,\lambda)=J(\pi_\lambda^*,\lambda)\le V^{\pi^*}(\mu)+\frac{\lambda C_\phi}{1-\gamma}
	\end{equation*}
	Minimizing $\lambda$ both sides, we gain $RHS\le J(\pi^*)=LHS$
	Simply putting the two results, we must have equality throughout.
\end{proof}

\end{appendix}

\newpage

\end{document}